\documentclass[twoside,11pt]{article}

\usepackage[utf8]{inputenc}
\usepackage{jmlr2e}
\usepackage{lastpage}
\usepackage{mathtools, amssymb}
\usepackage{enumerate}
\usepackage{xcolor}
\usepackage{hyperref}  
\usepackage{booktabs}
\usepackage{threeparttable}

\usepackage{lastpage}

\hypersetup{hidelinks}

\jmlrheading{24}{2023}{1-\pageref{LastPage}}{10/22}{2/23}{22-1191}{Chang hoon Song, Geonho Hwang, Jun ho Lee and Myungjoo Kang}
\ShortHeadings{title}{Song, Hwang, Lee and Kang}

\ShortHeadings{MINIMAL WIDTH FOR UNIVERSAL PROPERTY OF DEEP RNN}{Song, Hwang, Lee and Kang}
\firstpageno{1}

\newcommand{\cB}{\mathcal{B}}
\newcommand{\cC}{\mathcal{C}}

\newcommand{\cK}{\mathcal{K}}
\newcommand{\cL}{\mathcal{L}}

\newcommand{\cN}{\mathcal{N}}

\newcommand{\cP}{\mathcal{P}}
\newcommand{\cQ}{\mathcal{Q}}
\newcommand{\cR}{\mathcal{R}}

\newcommand{\cT}{\mathcal{T}}

\newcommand{\bN}{\mathbb{N}}

\newcommand{\bR}{\mathbb{R}}

\newcommand{\eps}{\epsilon}
\newcommand{\enc}{\operatorname{Enc}}
\newcommand{\mem}{\operatorname{Mem}}
\newcommand{\dec}{\operatorname{Dec}}

\newcommand{\minw}{w_{\text{min}}}
\newcommand{\ReLU}{\operatorname{ReLU}}
\newcommand{\sigmoid}{\sigma_{\operatorname{sig}}}

\newcommand{\review}[1]{{\color{black}#1}}
\newcommand{\reviewA}[1]{{\color{black}#1}}
\newcommand{\reviewB}[1]{{\color{black}#1}}
\newcommand{\reviewAB}[1]{{\color{black}#1}}

\title{Minimal Width for Universal Property of Deep RNN}

\author{\name Chang hoon Song \email goldbach2@snu.ac.kr \\
       \addr Department of Mathematical Science\\
       Seoul National University
       \AND
       \name Geonho Hwang \email hgh2134@snu.ac.kr \\
       \addr Department of Mathematical Science\\
       Seoul National University
       \AND
       \name Jun ho Lee \email jhlsat@kongju.ac.kr\\
       \addr Kongju National Universality
       \AND
       \name Myungjoo Kang \email mkang@snu.ac.kr \\
       \addr Department of Mathematical Science\\
       Seoul National University}
       
\editor{Mehryar Mohri}

\begin{document}

\maketitle

\begin{abstract}
A recurrent neural network (RNN) is a widely used deep-learning network for dealing with sequential data.
Imitating a dynamical system, an infinite-width RNN can approximate any open dynamical system in a compact domain.
In general, deep narrow networks with bounded width and arbitrary depth are more effective than wide shallow networks with arbitrary width and bounded depth in practice; however, the universal approximation theorem for deep narrow structures has yet to be extensively studied.
In this study, we prove the universality of deep narrow RNNs and show that the upper bound of the minimum width for universality can be independent of the length of the data.
Specifically, we show a deep RNN with ReLU activation can approximate any continuous function or $L^p$ function with the widths $d_x+d_y+3$ and $\max\{d_x+1,d_y\}$, respectively, where the target function maps a finite sequence of vectors in $\mathbb{R}^{d_x}$ to a finite sequence of vectors in $\mathbb{R}^{d_y}$.
We also compute the additional width required if the activation function is sigmoid or more.
In addition, we prove the universality of other recurrent networks, such as bidirectional RNNs.
Bridging a multi-layer perceptron and an RNN, our theory and technique can shed light on further research on deep RNNs.
\end{abstract}


\section{Introduction}
Deep learning, a type of machine learning that approximates a target function using a numerical model called an artificial neural network, has shown tremendous success
in diverse fields, such as regression \citep{garnelo2018neural}, image processing \citep{dai2021coatnet,zhai2022scaling}, speech recognition \citep{baevski2020wav2vec}, and natural language processing \citep{lecun2015deep, lavanya2021deep}.

While the excellent performance of deep learning is attributed to a combination of various factors, it is reasonable to speculate that its notable success is partially based on the \textit{universal approximation theorem}, which states that neural networks are capable of arbitrarily accurate approximations of the target function. 

Formally, for any given function $f$ in a target class and $\eps>0$, there exists a network $\cN$ such that $\|f - \cN\|< \eps$. 
In topological language, the theorem states that a set of networks is dense in the class of the target function.
In this sense, the closure of the network defines the range of functions it network can represent. 

As universality provides the expressive power of a network structure, studies on the universal approximation theorem have attracted increasing attention.
Examples include the universality of multi-layer perceptrons (MLPs), the most basic neural networks \reviewAB{\citep{Cybenko_1989, Hornik_1989, leshno1993multilayer}}, and the universality of recurrent neural networks (RNNs) with the target class of open dynamical systems \citep{Schafer_2007}.
Recently, \cite{Zhou_2020} has demonstrated the universality of convolutional neural networks.

Classical universal approximation theorems specialize in the representation power of shallow wide networks with bounded depth and unbounded width.
Based on mounting empirical evidence that deep networks demonstrate better performance than wide networks, the construction of deep networks instead of shallow networks has gained considerable attention in recent literature.  
Consequently, researchers have started to analyze the universal approximation property of deep networks \citep{cohen2016expressive, lin2018resnet, rolnick2017power, Kidger_2020}. 
Studies on MLP have shown that wide shallow networks require only two depths to have universality, while deep narrow networks require widths at least as their input dimension.

A wide network obtains universality by increasing its width even if the depth is only two \reviewAB{\citep{Cybenko_1989, Hornik_1989, leshno1993multilayer}}.
However, in the case of a deep network, there is a function for a narrow network that cannot be able to approximated, regardless of its depth \citep{Lu_2017, Park_2021}.
Therefore, clarifying the \textit{minimum width} to guarantee universality is crucial, and studies are underway to investigate its lower and upper bounds, narrowing the gap.

\begin{table}[t]
    \centering
    \label{table:summary_of_results}
    \begin{threeparttable}
        \begin{tabular}{cccc}
            \toprule
            Network & Function class & Activation & \multicolumn{1}{c}{Result}
            \\
            \midrule
            RNN & $C\left(\cK,\bR^{d_y}\right)$\tnote{$\dagger$} & ReLU & \reviewB{$\minw\le d_x+d_y+3$} \\
            & & conti. nonpoly\tnote{1} & $\minw\le d_x+d_y+3$ \\
            & & conti. nonpoly\tnote{2} & $\minw\le d_x+d_y+4$ \\ 
            & $L^p\left(\cK,\bR^{d_y}\right)$\tnote{$\dagger$} & conti. nonpoly\tnote{\reviewB{1},2} & \hspace{3em}$\minw\le \max\left\{d_x+1,d_y\right\}+1$ \\
            & \reviewB{$L^p\left(\bR^{d_x},\bR^{d_y}\right)$\tnote{$\dagger$}} & ReLU & $\minw = \max\left\{d_x+1,d_y\right\}$ \\
            \midrule
            LSTM & $C\left(\cK,\bR^{d_y}\right)$\tnote{$\dagger$} & & $\minw\le d_x+d_y+3$ \\
            & \reviewB{$L^p\left(\cK,\bR^{d_y}\right)$\tnote{$\dagger$}} & & \reviewB{$\minw\le \max\left\{d_x+1,d_y\right\}+1$} \\
            GRU & $C\left(\cK,\bR^{d_y}\right)$\tnote{$\dagger$} & & $\minw\le d_x+d_y+3$\\
            & \reviewB{$L^p\left(\cK,\bR^{d_y}\right)$\tnote{$\dagger$}} & & \reviewB{$\minw\le \max\left\{d_x+1,d_y\right\}+1$} \\
            \midrule
            BRNN & $C\left(\cK,\bR^{d_y}\right)$ & ReLU & \reviewB{$\minw\le d_x+d_y+3$} \\
             & & conti. nonpoly\tnote{1} & $\minw\le d_x+d_y+3$ \\
            & & conti. nonpoly\tnote{2} & $\minw\le d_x+d_y+4$ \\
            & $L^p\left(\cK,\bR^{d_y}\right)$ & conti. nonpoly\tnote{\reviewB{1},2} & \hspace{3em}$\minw\le \max\left\{d_x+1,d_y\right\}+1$ \\
            & \reviewB{$L^p\left(\bR^{d_x},\bR^{d_y}\right)$} & \reviewB{ReLU} & \reviewB{$\minw\le \max\left\{d_x+1, d_y\right\}$} \\
            \bottomrule
        \end{tabular}
        \begin{tablenotes}\footnotesize
            \item[$\dagger$] requires the class to consists of past-dependent functions.
            \item[1] requires an activation $\sigma$ to be continuously differentiable at some point $z_0$ with $\sigma(z_0)=0$ and $\sigma'(z_0)\ne0$. $\tanh$ belongs here.
            \item[2] requires an activation $\sigma$ to be continuously differentiable at some point $z_0$ with $\sigma'(z_0)\ne0$. A logistic sigmoid function belongs here.
        \end{tablenotes}
        \caption{\reviewA{Summary of our results on the upper bound of the minimal width $\minw$ for the universality of RNNs}. In the table, $\cK$ indicates a compact subset of $\bR^{d_x}$ and $1\le p<\infty$. We abbreviate continuous to ``conti''.}
    \end{threeparttable}
\end{table}

\textit{Recurrent neural networks} (RNNs) \citep{rumelhart1986learning, elman1990finding} have been crucial for modeling complex temporal dependencies in sequential data.
They have various applications in diverse fields, such as language modeling \citep{mikolov2010recurrent, jozefowicz2016exploring}, speech recognition \citep{graves2014towards, bahdanau2016end}, recommendation systems \citep{hidasi2015session, wu2017recurrent}, and machine translation \citep{bahdanau2014neural}. 
Deep RNNs are widely used and have been successfully applied in practical applications.
However, their theoretical understanding remains elusive despite their intensive use.
This deficiency in existing studies motivated our work.

In this study, we prove the universal approximation theorem of deep narrow RNNs and discover the upper bound of their minimum width.
The target class consists of a sequence-to-sequence function that depends solely on past information.
We refer to such functions as \textit{past-dependent} functions.
We provide the upper bound of the minimum width of the RNN for universality in the space of the past-dependent functions.
Surprisingly, the upper bound is independent of the length of the sequence.
This theoretical result highlights the suitability of the recurrent structure for sequential data compared with other network structures.
Furthermore, our results are not restricted to RNNs; they can be generalized to variants of RNNs, including \textit{long short-term memory} (LSTM), \textit{gated recurrent units} (GRU), and \textit{bidirectional RNNs} (BRNN).
As corollaries of our main theorem, LSTM and GRU are shown to have the same universality and target class as an RNN.
We also prove that the BRNN can approximate any sequence-to-sequence function in a continuous or $L^p$ space under the respective norms.
We also present the upper bound of the minimum width for these variants.
Table \ref{table:summary_of_results} outlines our main results.

\subsection{Summary of Contributions}
With a target class of functions that map a finite sequence \reviewA{$x\in\bR^{d_x\times N}$} to a finite sequence \reviewA{$y\in\bR^{d_y\times N}$}, we prove the following:
\begin{itemize}
    \item A deep RNN can approximate any past-dependent sequence-to-sequence continuous function with \reviewB{width $d_x+d_y+3$ for the ReLU or $\tanh$\footnotemark[1]}, and $d_x+d_y+4$ for non-degenerating activations.
    \item A deep RNN can approximate any past-dependent $L^p$ function ($1\le p<\infty$) with width $\max\left\{d_x+1, d_y\right\}$ for the ReLU activation and $\max\left\{d_x+1, d_y\right\}+1$ for non-degenerating activations.
    \item A deep BRNN can approximate any sequence-to-sequence continuous function with width \reviewB{$d_x+d_y+3$ for the ReLU or $\tanh$\footnotemark[1]}, and $d_x+d_y+4$ for non-degenerating activations.
    \item A deep BRNN can approximate any sequence-to-sequence $L^p$ function ($1\le p<\infty$) with width $\max\left\{d_x+1,d_y\right\}$ for the ReLU activation and $\max\left\{d_x+1,d_y\right\}+1$ for non-degenerating activations.
    \item A deep LSTM or GRU can approximate any past-dependent sequence-to-sequence continuous function with width $d_x+d_y+3$ and $L^p$ function with width $\max\left\{d_x+1,d_y\right\}+1$.
\end{itemize}
\footnotetext[1]{Generally, non-degenerate $\sigma$ with $\sigma(z_0)=0$ requires the same minimal width as $\tanh$.}

\subsection{Organization}
In Section 2, we briefly review previous studies on the universality of neural networks.
Section 3 provides some notations, network configuration, a description of the target function class, and the condition of the activation function.
Section 4 addresses the approximation property in a continuous function space.
First, we fit only the last component of an output sequence and extend it to all components of the sequence.
In Section 5, we present the upper bound of the minimum width in $L^p$ space.
As an expansion of our theorems, the universal approximation theorems for the LSTM, GRU, and BRNN are deduced in Section 6.
Sections 7 and 8 present a discussion and conclusions, respectively.


\section{Related Works}
\label{section:related_works}
We briefly review some of the results of studies on the universal approximation property.
Studies have been conducted to determine whether a neural network can learn a sufficiently wide range of functions, that is, whether it has universal properties.
\cite{Cybenko_1989} and \cite{Hornik_1989} first proved that the most basic network, a simple two-layered MLP, can approximate arbitrary continuous functions defined on a compact set.
Some follow-up studies have investigated the universal properties of other structures for a specific task, such as a convolutional neural network for image processing \citep{Zhou_2020}, an RNN for open dynamical systems \citep{Schafer_2007,Hanson_2020}, and transformer networks for translation and speech recognition \citep{Yun_2020}.
Particularly for RNNs, \cite{Schafer_2007} showed that an open dynamical system with continuous state transition and output function can be approximated by a network with a wide RNN cell and subsequent linear layer in finite time.
\cite{Hanson_2020} showed that the trajectory of the dynamical system can be reproduced with arbitrarily small errors up to infinite time, assuming a stability condition on long-term behavior.

While such prior studies mainly focused on wide and shallow networks, several studies have determined whether a deep narrow network with bounded width can approximate arbitrary functions in some kinds of target function classes, such as a space of continuous functions or an $L^p$ space \citep{Lu_2017, hanin2017approximating, Johnson_2019, Kidger_2020, Park_2021}.
Unlike the case of a wide network that requires only one hidden layer for universality, there exists a function that cannot be approximated by any network whose width is less than a certain threshold.
\reviewB{Specifically, in a continuous function space $C\left(K,\bR^{d_y}\right)$ with the supreme-norm $\left\|f\right\|_\infty\coloneqq\sup_{x\in K}\left\|f(x)\right\|$ on a compact subset $K\subset\bR^{d_x}$, \cite{hanin2017approximating} showed negative and positive results that MLPs do not attain universality with width $d_x$.
Still, width $d_x+d_y$ is sufficient for MLPs to achieve universality with the ReLU activation function.
\cite{Johnson_2019, Kidger_2020} generalized the condition on the activation and proved that $d_x$ is too narrow, but $d_x+d_y+2$ is sufficiently wide for the universality in $C\left(K,\bR^{d_y}\right)$.
On the other hand, in an $L^p$ space $L^p\left(\bR^{d_x},\bR^{d_y}\right)$ with $L^p$-norm $\left\|f\right\|_p\coloneqq\left(\int_{\bR^{d_x}} f(x)\,dx\right)^{1/p}$ for $p\in[0,\infty)$, \cite{Lu_2017} showed that the width $d_x$ is insufficient for MLPs to have universality, but $d_x+4$ is enough, with the ReLU activation when $p=1$ and $d_y=1$.
\cite{Kidger_2020} found that MLPs whose width is $d_x+d_y+1$ achieve universality in $L^p\left(\bR^{d_x},\bR^{d_y}\right)$ with the ReLU activation, and \cite{Park_2021} determined the exact value $\max\left\{d_x+1,d_y\right\}$ required for universality in the $L^p$ space with ReLU.}

As described earlier, studies on deep narrow MLP have been actively conducted, but the approximation ability of deep narrow RNNs remains unclear.
This is because the process by which the input affects the result is complicated compared with that of an MLP.
The RNN cell transfers information to both the next time step in the same layer and the same time step in the next layer, which makes it difficult to investigate the minimal width.
In this regard, we examined the structure of the RNN to apply the methodology and results from the study of MLPs to deep narrow RNNs.

 
\section{Terminologies and Notations}
\label{section:terminologies_and_notations}
This section introduces the definition of network architecture and the notation used throughout this paper.
$d_x$ and $d_y$ denote the dimension of input and output space, respectively.
$\sigma$ is an activation function unless otherwise stated.
Sometimes, $v$ indicates a vector with suitable dimensions.

First, we used square brackets, subscripts, and colon symbols to index a sequence of vectors.
More precisely, for a given sequence of $d_x$-dimensional vectors $x:\bN\rightarrow\bR^{d_x}$, $x[t]_j$ or $x_j[t]$ denotes the $j$-th component of the $t$-th vector.
The colon symbol $:$ is used to denote a continuous index, such as $x[a:b]=\left(x[i]\right)_{a\le i\le b}$ or $x[t]_{a:b}=\left(x[t]_a, x[t]_{a+1},\ldots,x[t]_b\right)^T\in\bR^{b-a+1}$.
We call the sequential index $t$ by time and each $x[t]$ a \textit{token}.

Second, we define the token-wise linear maps $\cP:\bR^{d_x\times N}\rightarrow\bR^{d_s\times N}$ and $\cQ:\bR^{d_s\times N}\rightarrow\bR^{\reviewAB{d_y}\times N}$ to connect the input, hidden state, and output space.
As the dimension of the hidden state space $\bR^{d_s}$ on which the RNN cells act is different from those of the input domain $\bR^{d_x}$ and output domain $\bR^{d_y}$, we need maps adjusting the dimensions of the spaces.
For a given matrix $P\in\bR^{d_s\times d_x}$, a \textit{lifting map} $\cP(x)[t]\coloneqq Px[t]$ lifts the input vector to the hidden state space.
Similarly, for a given matrix $Q\in \bR^{d_y\times d_s}$, a \textit{projection map} $\cQ(s)[t]\coloneqq Qs[t]$ projects a hidden state onto the output vector.
As the first token defines a token-wise map, we sometimes represent token-wise maps without a time length, such as $\cP:\bR^{d_x}\rightarrow\bR^{d_s}$ instead of $\cP:\bR^{d_x\times N}\rightarrow\reviewB{\bR^{d_s\times N}}$.

Subsequently, an RNN is constructed using a composition of basic recurrent cells between the lifting and projection maps.
We considered four basic cells: RNN, LSTM, GRU, and BRNN.

\begin{itemize}
    \item \textbf{RNN Cell}
        A \textit{recurrent cell}, \textit{recurrent layer}, or \textit{RNN cell} $\cR$ maps an input sequence $x=\left(x[1],x[2],\ldots\right)=\left(x[t]\right)_{t\in\bN}\in\bR^{d_s\times \bN}$ to an output sequence $y=\left(y[t]\right)_{t\in\bN}\in\bR^{d_s\times\bN}$ using
        \begin{equation}
            \label{eqn:rnn}
            y[t+1]=\cR(x)[t+1]=\sigma\left(A\cR(x)[t]+Bx[t+1]+\theta\right),
        \end{equation}
        where $\sigma$ is an activation function, $A,B\in\bR^{d_s\times d_s}$ are the weight matrices, and $\theta\in\bR^{d_s}$ is the bias vector.
        The initial state \reviewA{$\cR(x)[0]$} can be an arbitrary constant vector, which is zero vector $\mathbf{0}$ in this setting.
    \item \reviewB{\textbf{LSTM cell}
        \textit{An LSTM cell} is an RNN cell variant with an internal cell state and gate structures.
        As an original version proposed in \cite{hochreiter1997long}, an LSTM cell includes only two gates: an input gate and an output gate.
        Still, the base of almost all networks used in modern deep learning is LSTM cells with an additional gate called the forget gate, introduced in \cite{gers2000learning}.
        Hence we present the LSTM cell with three gates instead of the original one.
        
        Mathematically, an LSTM cell $\cR_{LSTM}$ is a process that computes an output $h\in\bR^{d_s}$ and cell state $c\in\bR^{d_s}$, defined by the following relation:
        \begin{equation}
            \label{eqn:lstm}
            \begin{aligned}
                i[t+1]&=\sigmoid\left(W_ix[t+1]+U_ih[t]+b_i\right)\\
                f[t+1]&=\sigmoid\left(W_fx[t+1]+U_fh[t]+b_f\right)\\
                g[t+1]&=\tanh\left(W_gx[t+1]+U_gh[t]+b_g\right)\\
                o[t+1]&=\sigmoid\left(W_ox[t+1]+U_oh[t]+b_o\right)\\
                c[t+1]&=f[t+1]\odot c[t]+i[t+1]\odot g[t+1]\\
                h[t+1]&=o[t+1]\odot\tanh\left(c[t+1]\right)
            \end{aligned}
        \end{equation}
        where $W_*\in\bR^{d_s\times d_s}$ and $U_*\in\bR^{d_s\times d_s}$ are weight matrices; $b_*\in\bR^{d_s}$ is the bias vector for each $*=i,f,g,o$; and $\sigmoid$ is the sigmoid activation function.
        $\odot$ denotes component-wise multiplication, and we set the initial state to zero in this study.

        Although some variants of the LSTM have additional connections, such as peephole connection, our result for LSTM extends naturally.}
    \item \textbf{GRU cell}
        Another import variation of an RNN is GRU, proposed by \cite{cho2014properties}.
        Mathematically, a \textit{GRU cell} $\cR_{GRU}$ is a process that computes $h\reviewB{\in\bR^{d_s}}$ defined by
        \begin{equation}
            \label{eqn:gru}
            \begin{aligned}
                r[t+1]&=\sigmoid\left(W_rx[t+1]+U_rh[t]+b_r\right),\\
                \tilde{h}[t+1]&=\tanh\left(W_hx[t+1]+U_h\left(r[t+1]\odot h[t]\right)+b_h\right),\\
                z[t+1]&=\sigmoid\left(W_zx[t+1]+U_zh[t]+b_z\right),\\
                h[t+1]&=\left(1-z[t+1]\right)\reviewB{\odot} h[t]+z[t+1]\reviewB{\odot} \tilde{h}[t+1],
            \end{aligned}
        \end{equation}
        where $W_*\in\bR^{d_s\times \review{d_s}}$ and $U_*\in\bR^{d_s\times d_s}$ are weight matrices, $b_*\in\bR^{d_s}$ is the bias vector for each $*=r,z,h$, and $\sigmoid$ is the sigmoid activation function. 
        $\odot$ denotes component-wise multiplication, and we set the initial state to zero in this study.
    \item \textbf{BRNN cell}
        A \textit{BRNN cell} $\cB\cR$ consists of a pair of RNN cells and a token-wise linear map that follows the cells \citep{Schuster_1997}.
        An RNN cell $\cR_1$ in the BRNN cell $\cB\cR$ receives input from $x[1]$ to $x[N]$ and the other $\cR_2$ receives input from $x[N]$ to $x[1]$ in reverse order.
        Then, the linear map $\cL$ in $\cB\cR$ combines the two outputs from the RNN cells.
        Specifically, a BRNN cell $\cB\cR$ is defined as follows:
        \begin{equation}
            \label{eqn:brnn}
            \begin{aligned}
                \cR(x)[t+1]&\coloneqq \sigma\left(A\cR_1(x)[t]+Bx[t+1]+\theta\right),\\
                \bar{\cR}(x)[t-1]&\coloneqq \sigma\left(\bar{A}\bar{\cR}(x)[t]+\bar{B}x[t-1]+\bar{\theta}\right),\\
                \cB\cR(x)[t]&\coloneqq \cL\left(\cR(x)[t],\bar{\cR}(x)[t]\right)\\
                &\coloneqq W\cR(x)[t]+\bar{W}\bar{\cR}(x)[t].
            \end{aligned}
        \end{equation}
        where $A$, $B$, $\bar{A}$, $\bar{B}$, $W$, and $\bar{W}$ are weight matrices; $\theta$ and $\bar{\theta}$ are bias vectors.
    \item \textbf{Network architecture}
        An \textit{RNN} $\cN$ comprises a lifting map $\cP$, projection map $\cQ$, and $L$ recurrent cells $\cR_1,\ldots,\cR_L$;
        \begin{equation}
            \label{eqn:RNN}
            \cN\coloneqq \cQ\circ\cR_L\circ\cdots\circ\cR_1\circ\cP.
        \end{equation}
        We denote the network as a \textit{stack RNN} or \textit{deep RNN} when $L\ge 2$, and each output of the cell $\cR_i$ as the $i$-th hidden state.
        $d_s$ indicates the width of the network.
        \review{
            As a special case, a recurrent cell $\cR$ in \eqref{eqn:rnn} is a composition of activation and a token-wise affine map with $A=0$; A token-wise MLP is a specific example of RNN.
        }
        If LSTM, GRU, or BRNN cells replace recurrent cells, the network is called an LSTM, a GRU, or a BRNN.
\end{itemize}

In addition to the type of cell, the activation function $\sigma$ affects universality.
We focus on the case of ReLU or $\tanh$ while also considering the general activation function satisfying the condition proposed by \citep{Kidger_2020}.
$\sigma$ is a continuous non-polynomial function that is continuously differentiable at some $z_0$ with $\sigma'(z_0)\ne0$.
We refer to the condition as a \textit{non-degenerate} condition and $z_0$ as a \textit{non-degenerating point}.

Finally, the target class must be set as a subset of the sequence-to-sequence function space, from $\bR^{d_x}$ to $\bR^{d_y}$.
Given an RNN $\cN$, each token $y[t]$ of the output sequence $y=\cN(x)$ depends only on $x[1\mathbin{:}t]\coloneqq\left(x[1],x[2],\ldots,x[t]\right)$ for the input sequence $x$. 
We define this property as \textit{past dependency} and a function with this property as a \textit{past-dependent} function.
More precisely, if all the output tokens of a sequence-to-sequence function are given by $f[t]\left(x[1\mathbin{:}t]\right)$ for functions $f[t]:\bR^{d_x\times t}\rightarrow \bR^{d_y}$, we say that the function is past-dependent.
Meanwhile, we must fix the finite length or terminal time $N<\infty$ of the input and output sequence. 
Without additional assumptions such as in \citep{Hanson_2020}, errors generally accumulate over time, making it impossible to approximate
implicit dynamics up to infinite time regardless of past dependency.
Therefore we set the target function class as a class of past-dependent sequence-to-sequence functions with sequence length $N$.

\begin{remark}
    On a compact domain and under bounded length, the continuity of $f:\bR^{d_x\times N}\rightarrow \bR^{d_y\times N}$ implies that of each $f[t]:\bR^{d_x\times t}\rightarrow \bR^{d_y}$ and vice versa.
    In the case of the $L^p$ norm with $1\le p<\infty$, $f:\bR^{d_x\times N}\rightarrow \bR^{d_y\times N}$ is $L^p$ integrable if and only if $f[t]$ is $L^p$ integrable for each $t$.
    In both cases, the sequence of functions $\left(f_n\right)_{n\in\bN}$ converges to $g$ if and only if $\left(f_n[t]\right)_{n\in\bN}$ converges to $g[t]$ for each $t$.
    Thus, we focus on approximating $f[t]$ for each $t$ under the given conditions.
\end{remark}
\reviewA{
\begin{remark}
    There are typical tasks where the length of the output sequence differs from that of the input sequence.
    In translation tasks, for instance, inputs and outputs are sentences represented as sequences of tokens of words in different languages, and the length of each sequence differs in general.
    Nevertheless, even in such a case, a sequence $x=\left(x[1],x[2],\ldots,x[N]\right)\in\bR^{d_x\times N}$ can be naturally embedded into $\bR^{d_x\times M}$ for an integer $M>N$ as an extended sequence $\left(x[1],x[2],\ldots,x[N],\mathbf{0},\ldots,\mathbf{0}\right)$ or $\left(\mathbf{0},\ldots,\mathbf{0},x[1],x[2],\ldots,x[N]\right)$.
    By the natural embedding, a map between sequences $f:\bR^{d_x\times N_1}\rightarrow\bR^{d_y\times N_2}$ is a special case of a map between $f:\bR^{d_x\times M}\rightarrow\bR^{d_y\times M}$ for an integer $M\ge N_1,N_2$.
\end{remark}
}
Sometimes, only the last value $\cN(x)[N]$ is required considering an RNN $\cN$ as a sequence-to-vector function $\cN:\bR^{d_x\times N}\rightarrow\bR^{d_y}$.
We freely use the terminology RNN for sequence-to-sequence and sequence-to-vector functions because there is no confusion when the output domain is evident.

We have described all the concepts necessary to set a problem, but we end this section with an introduction to the concepts used in the proof of the main theorem. 
For the convenience of the proof, we slightly modify the activation $\sigma$ to act only on some components, instead of all components.
With activation $\sigma$ and index set $I\subseteq \bN$, \textit{the modified activation $\sigma_I$} is defined as
\begin{equation}
    \label{eqn:modified_activation}
    \sigma_I(s)_i=\left\{\begin{array}{ll}\sigma(s_i)&\text{ if $i\in I$ }\\ s_i&\text{ otherwise } \end{array}\right.
\end{equation}
Using the modified activation function $\sigma_I$, the basic cells of the network are modified in \eqref{eqn:rnn}.
For example, a \textit{modified recurrent cell} can be defined as
\begin{equation}
    \label{eqn:modified_rnn}
    \begin{aligned}
        \cR(x)[t+1]_i&=\sigma_I\left(A\cR(x)[t]+Bx[t+1]+\theta\right)_i\\
        &=\left\{\begin{array}{ll}\sigma\left(A\cR(x)[t]+Bx[t+1]+\theta\right)_i &\text{ if }i\in I\\
                                    \left(A\cR(x)[t]+Bx[t+1]+\theta\right)_i &\text{ otherwise }\end{array}\right..
    \end{aligned}
\end{equation}
Similarly, \textit{modified RNN, LSTM, GRU, or BRNN} is defined using modified cells in \eqref{eqn:rnn}.
This concept is similar to the enhanced neuron of \cite{Kidger_2020} in that activation can be selectively applied, but is different in that activation can be applied to partial components.

As activation leads to the non-linearity of a network, modifying the activation can affect the minimum width of the network.
Fortunately, the following lemma shows that the minimum width increases by at most one owing to the modification.
We briefly introduce the ideas here, with a detailed proof provided in the appendix.

\begin{lemma}
    \label{lemma:modified_rnn}
    Let $\bar{\cR}:\bR^{d\times N}\rightarrow\bR^{d\times N}$ be a modified RNN cell, $\bar{\cQ}:\bR^{d}\rightarrow\bR^{d}$, and $\bar{\cP}:\bR^{d}\rightarrow\bR^{d}$ be a token-wise linear projection and lifting map.
    Suppose that an activation function $\sigma$ of $\bar{\cR}$ is non-degenerate with a non-degenerating point $z_0$.
    Then for any compact subset $K\subset\bR^{d}$ and $\eps>0$, there exists RNN cells $\cR_1$, $\cR_2:\bR^{(d+\beta(\sigma))\times N}\rightarrow\bR^{(d+\beta(\sigma))\times N}$, and a token-wise linear map $\cP:\bR^d\rightarrow\bR^{d+\beta(\sigma)}$, $\cQ:\bR^{d+\beta(\sigma)}\rightarrow\bR^d$ such that
    \begin{equation}
        \sup_{x\in K^N}\left\|\bar{\cQ}\circ\bar{\cR}\circ\bar{\cP}(x)-\cQ\circ\cR_2\circ\cR_1\circ\cP(x)\right\|<\eps,
    \end{equation} 
    where $\beta(\sigma)=\left\{\begin{array}{ll}0&\text{ if \reviewB{$\sigma\left(z_0\right)=0$}}\\ 1&\text{ otherwise}\end{array}\right.$
\end{lemma}
\begin{proof}[Sketch of proof]
    The detailed proof is available in Appendix \ref{appendix:modified_rnn}.
    We use the Taylor expansion of $\sigma$ at $z_0$ to recover the value before activation.
    For the $i$-th component with $i\not\in I$, choose a small $\delta>0$ and linearly approximate $\sigma\left(z_0+\delta z\right)$ as $\sigma(z_0)+\delta\sigma'(z_0)z$. An affine transform after the Taylor expansion recovers $z$.
\end{proof}

\reviewB{\begin{remark}
    Since the additional width is only used to move the input of each RNN cell to near $z_0$, it seems easily removable using a bias term at first glance, provided the projection map is affine instead of linear.
    However, a bias term $\theta$ in \eqref{eqn:rnn} is fixed across all time steps, while we need different translations: one for the first token and the other for tokens after the second.
    
    For example, for an RNN cell $\cR^\delta:\bR^{1\times\bN}\rightarrow\bR^{1\times\bN}$ defined by 
    \begin{equation}
        \cR^\delta(x)[t+1]=\sigma\left(a\cR^\delta(x)[t]+\delta x[t+1]+\theta\right),
    \end{equation}
    we need $\theta=z_0$ to use the Taylor expansion at $z_0$ to the first token $\cR^\delta(x)[1]=\sigma\left(\delta x[1]+\theta\right)$.
    When $\theta=z_0$, we cannot represent the second token 
    \begin{align}
        \cR^\delta(x)[2]
        &= \sigma\left(a\cR^\delta(x)[1]+\delta x[2]+z_0\right)
        \\
        &= \sigma\left(a\left(\sigma\left(z_0\right)+\delta\sigma'\left(z_0\right)x[1]+o\left(\delta\right)\right)+\delta x[2]+z_0\right),
    \end{align}
    as the Taylor expansion at $z_0$ unless $a\sigma\left(z_0\right)=0$.
    
    In other words, the need for the additional width originates from that previous state $\cR(x)[t-1]$ appears only after the second token.
    Nonetheless, we can make $\beta\equiv0$, by setting the proper initial state for each RNN cell and using an affine projection map.
\end{remark}}

The lemma implies that a modified RNN can be approximated by an RNN with at most one additional width.
For a given modified RNN $\bar{\cQ}\circ\bar{\cR}_L\circ\cdots\circ\bar{\cR}_1\circ\review{\bar{\cP}}$ of width $d$ and $\eps>0$, we can find RNN $\cR_1,\ldots,\cR_{2L}$ and linear maps $\cP_1,\ldots,\cP_L$, $\cQ_1,\ldots,\cQ_L$ such that
\begin{equation}
    \sup_{x\in K^N}\left\|\bar{\cQ}\circ\bar{\cR}_L\circ\cdots\circ\bar{\cR}_1\circ\bar{\cP}(x)-
    \left(\cQ_L\cR_{2L}\cR_{2L-1}\cP_L\right)\circ\cdots\circ\left(\cQ_1\cR_2\cR_1\cP_1\right)(x)\right\|<\eps.
\end{equation}
The composition $\cR\circ\cP$ of an RNN cell $\cR$ and token-wise linear map $\cP$ can be substituted by another RNN cell $\cR'$.
More concretely, for $\cR$ and $\cP$ defined by
\begin{align}
    \cR(x)[t+1]&=\sigma\left( A\cR(x)[t]+Bx[t+1]+\theta \right),
    \\
    \cP(x)[t]&=P\left(x[t]\right),
    \intertext{$\cR\circ\cP$ defines an RNN cell $\cR'$}
    \cR'(x)[t+1]&=\sigma\left( A\cR'(x)[t]+BPx[t+1]+\theta\right).
\end{align}
Thus, $\cR_{2l+1}\left(\cP_{l+1}\cQ_l\right)$ becomes a recurrent cell, and $\left(\cQ_L\cR_{2L}\cR_{2L-1}\cP_L\right)\circ\cdots\circ\left(\cQ_1\cR_2\cR_1\cP_1\right)(x)$ defines a network of form \eqref{eqn:RNN}.


\section{Universal Approximation for Deep RNN in Continuous Function Space}
\label{section:universal_continuous}
This section introduces the universal approximation theorem of deep RNNs in continuous function space.
\begin{theorem}[Universal approximation theorem of deep RNN 1]
    \label{thm:universal_seq_to_seq}
    Let $f:\bR^{d_x\times N}\rightarrow\bR^{d_y\times N}$ be a continuous past-dependent sequence-to-sequence function and $\sigma$ be a non-degenerate activation function.
    Then, for any $\eps>0$ and compact subset $K\subset\bR^{d_x}$, there exists a deep RNN $\cN$ of width \reviewB{$d_x+d_y+3+\alpha(\sigma)$} such that
    \begin{equation}
        \sup_{x\in K^N}\sup_{1\le t\le N}\left\|f(x)[t]-\cN(x)[t]\right\|<\eps,
    \end{equation}
    \reviewB{where $\alpha\left(\sigma\right)=\begin{cases}
        0&\text{$\sigma$ is ReLU or a non-degenerating function with $\sigma\left(z_0\right)=0$.}\\
        1&\text{$\sigma$ is a non-degenerating function with $\sigma\left(z_0\right)\ne0$}.
    \end{cases}$}
\end{theorem}

To prove the above theorem, we deal with the case of the sequence-to-vector function $\cN:\bR^{d_x\times N}\rightarrow\bR^{d_y}$ first, in subsection \ref{subsection:seq_to_vec}.
Then, we extend our idea to a sequence-to-sequence function using bias terms to separate the input vectors at different times in subsection \ref{subsection:seq_to_seq}, and the proof of Theorem \ref{thm:universal_seq_to_seq} is presented at the end of this section.
As the number of additional components required in each step in the subsections depends on the activation function, we use $\alpha(\sigma)$ to state the theorem briefly.

\reviewA{\subsection{Approximation of sequence-to-vector function}}
\label{subsection:seq_to_vec}
\reviewB{The motivation for approximating a sequence-to-vector function $f:\bR^{d_x\times N}\rightarrow\bR^{d_y}$ is to copy a two-layered MLP via a modified RNN.
Note that the output of a two-layered MLP is a linear sum of outputs of its hidden nodes, and each output of the hidden nodes is a linear sum of inputs with the following activation function.
In Lemma \ref{lemma:linear_sum_into_rnn}, we construct a modified RNN that simulates a hidden node of an MLP, which can be represented as the sum of the inner product of some matrices and $N$ input vectors in $\bR^{d_x}$ with activation.
After that, we use an additional buffer component in Lemma \ref{lemma:mlp_in_rnn} to copy another hidden node in the two-layered MLP and the following modified RNN cell accumulates two results from the nodes.
The buffer component of the modified RNN cell is then reset to zero to copy another hidden node.
Repeating the procedure, the modified RNN with bounded width copies the two-layered MLP.}

Now, we present the statements and sketches of the proof corresponding to each step.
The following lemma implies that a modified RNN \reviewA{can} compute the linear sum of all the input components, which copies the hidden node of a two-layered MLP.
\begin{lemma}
    \label{lemma:linear_sum_into_rnn}
    Suppose $A[1],A[2],\cdots,A[N]\in\bR^{1\times d_x}$ are the given matrices. Then there exists a modified RNN  $\cN=\reviewB{\cR_N}\circ\cR_{N-1}\circ\cdots\circ\cR_1\circ\cP:\bR^{d_x\times N}\rightarrow \bR^{(d_x+1)\times N}$ of width $d_x+1$ such that (the symbol $\ast$ indicates that there exists some value irrelevant to the proof)
    \begin{equation}
        \label{eqn:linear_sum_into_rnn}
        \begin{aligned}
            \cN(x)[t]&=\begin{bmatrix}x[t]\\ \ast\end{bmatrix}&\text{ for }t<N,\\
            \cN(x)[N]&=\begin{bmatrix}x[N]\\ \sigma\left(\sum_{t=1}^NA[t]x[t]\right)\end{bmatrix}.
        \end{aligned}
    \end{equation}
\end{lemma}

\begin{proof}
    [Sketch of the proof]
    The detailed proof is available in Appendix \ref{appendix:linear_sum_into_rnn}.
    Define the $m$-th modified RNN cell $\cR_m$, of the form of \eqref{eqn:rnn} without activation, with $A_m=\begin{bmatrix}
    O_{d_x\times d_x}&O_{d_x\times 1}\\
    O_{1\times d_x}&1\end{bmatrix}$, $B_m=\begin{bmatrix}
    I_{d_x}&O_{d_x\times 1}\\
    b_m&\review{1}\end{bmatrix}$ where $b_m\in\bR^{1\times b_x}$.
    Then, the $(d_x+1)$th component $y[N]_{d_x+1}$ of the final output $y[N]$ after $N$ layers becomes a linear combination of $b_ix[j]$ with some constant coefficients $\alpha_{i,j}$ and $\sum_{i=1}^N\sum_{j=1}^N\alpha_{i,j}b_ix[j]$.
    Thus the coefficient of $x[j]$ is represented by $\sum_{i=1}^N\alpha_{i,j}b_i$, which we wish to be $A[j]$ for each $j=1,2,\ldots,N$.
    In matrix formulation, we intend to find $b$ satisfying $\Lambda^Tb=A$, where $\Lambda=\left\{\alpha_{i,j}\right\}_{1\le i,j\leq N }\in\bR^{N\times N}$, $b=\begin{bmatrix}b_1\\\vdots\\b_N\end{bmatrix}\in\bR^{N\times d_x}$, and $A=\begin{bmatrix}A[1]\\\vdots\\ A[N]\end{bmatrix}$.
    As $\Lambda$ is invertible there exist $b_i$ that solve $\left(\Lambda^T b\right)_j=A[j]$.
\end{proof}

After copying a hidden node using the above lemma, we add a component, $(d_x+2)$th, to copy another hidden node.
Then the results are accumulated in the $(d_x+1)$th component, and the final component is to be reset to copy another node.
As the process is repeated, a modified RNN replicates the output node of a two-layered MLP.
\begin{lemma}
    \label{lemma:mlp_in_rnn}
    Suppose $w_i\in\bR$, $A_i[t]\in\bR^{1\times d_x}$ are given for $t=1,2,\ldots,N$ and $i=1,2,\ldots,M$.
    Then, there exists a modified RNN $\cN:\bR^{d_x\times N}\rightarrow \bR$ of width $d_x+2$ such that
    \begin{equation}
        \label{eqn:mlp_in_rnn}
        \cN(x)=\sum_{i=1}^M w_i\sigma\left(\sum_{t=1}^N A_i[t]x[t]\right).
    \end{equation}
\end{lemma}
\begin{proof}
    First construct a modified RNN $\cN_1:\bR^{d_x\times N}\rightarrow\bR^{(d_x+2)\times N}$ of width $d_x+2$ such that
    \begin{align}
        \cN_1(x)[t]&=\begin{bmatrix}x[t]\\ \ast\\0\end{bmatrix} &\text{ for }t<N,\\
        \cN_1(x)[N]&=\begin{bmatrix}x[N]\\ \sigma\left(\sum_{t=1}^NA_1[t]x[t]\right)\\0\end{bmatrix},
    \end{align}
    as Lemma \ref{lemma:linear_sum_into_rnn}.
    Note that the final component does not affect the first linear summation and remains zero.
    Next, using the components except for the $(d_x+1)$th one, construct $\cN_2:\bR^{(d_x+2)\times N}\rightarrow\bR^{(d_x+2)\times N}$, which satisfies
    \begin{align}
        \cN_2\cN_1(x)[t]&=\begin{bmatrix}x[t]\\ \ast\\ \ast\end{bmatrix} &\text{ for }t<N,\\
        \cN_2\cN_1(x)[N]&=\begin{bmatrix}x[N]\\ \sigma\left(\sum_{t=1}^NA_1[t]x[t]\right)\\ \sigma\left(\sum_{t=1}^NA_2[t]x[t]\right)\end{bmatrix},
    \end{align}
    and use one modified RNN cell $\cR$ after $\cN_2$ to add the results and reset the last component:
    \begin{align}
        \cR\cN_2\cN_1(x)[t] &=\begin{bmatrix}x[t]\\ \ast\\ 0\end{bmatrix},
        \\
        \cR\cN_2\cN_1(x)[N] &=\begin{bmatrix}x[N]\\ w_1\sigma\left(\sum A_1[t]x[t]\right)+w_2\sigma\left(\sum A_2[t]x[t]\right)\\0\end{bmatrix}.
    \end{align}
    As the $(d_x+2)$th component is reset to zero, we use it to compute the third sum $w_3\sigma\left(\sum A_3[t]x[t]\right)$ and repeat until we obtain the final network $\cN$ such that
    \begin{equation}
        \cN(x)[N] = \begin{bmatrix}x[N]\\\sum_{i=1}^Mw_i\sigma\left(\sum_{t=1}^NA_i[t]x[t]\right)\\0\end{bmatrix}.
    \end{equation}
\end{proof}

\begin{remark}
    \label{remark:general_dy}
    The above lemma implies that a modified RNN of width $d_x+2$ can copy the output node of a two-layered MLP.
    We can extend this result to an arbitrary $d_y$-dimensional case.
    Note that the first $d_x$ components remain fixed, the $(d_x+1)$th component computes a part of the linear sum approximating the target function, and the $(d_x+2)$th component computes another part and is reset.
    When we need to copy another output node for another component of the output of the target function $f:\bR^{d_x\times N}\rightarrow\bR^{d_y\times N}$, only one additional width is sufficient.
    Indeed, the $(d_x+2)$th component computes the sum and the final component, and the $(d_x+3)$th component acts as a buffer to be reset in that case.
    By repeating this process, we obtain $(d_x+d_y+1)$-dimensional output from the modified RNN, which includes all $d_y$ outputs of the MLP and the components from the $(d_x+1)$th to the $(d_x+d_y)$th ones.
\end{remark}

\begin{theorem}[Universal approximation theorem of deep RNN 2]
    \label{thm:universal_seq_to_vec}
    Suppose a target $f:\bR^{d_x\times N}\rightarrow \bR^{d_y}$ is a continuous sequence-to-vector function, $K\subset\bR^{d_x}$ is a compact subset, $\sigma$ is a non-degenerating activation function, and $z_0$ is the non-degenerating point.
    Then, for any $\eps>0$, there exists a deep RNN $\cN:\bR^{d_x\times N}\rightarrow \bR^{d_y}$ of width $d_x+d_y+1+\beta(\sigma)$ such that
    \begin{equation}
        \sup_{x\in K^N}\left\|f(x)-\cN(x)\right\|<\eps,
    \end{equation}
    where $\beta(\sigma)=\left\{\begin{array}{ll}0&\text{ if \reviewB{$\sigma\left(z_0\right)=0$}}\\ 1&\text{ otherwise}\end{array}\right.$
\end{theorem}
\begin{proof}
    We present the proof for $d_y=1$ here, but adding $d_y-1$ width for each output component works for the case $d_y>1$.
    By the universal approximation theorem of the MLP, \reviewAB{Theorem 1 in \cite{leshno1993multilayer}}, there exist $w_i$ and $A_i[t]$ for $i=1,\ldots,M$ such that
    \begin{equation}
        \sup_{x\in K^N}\left\|f(x)-\sum_{i=1}^Mw_i\sigma\left(\sum_{t=1}^N A_i[t]x[t]\right)\right\|<\frac{\eps}{2}.
    \end{equation}
    Note that there exists a modified RNN $\bar{\cN}:\bR^{d_x\times N}\rightarrow\bR$ of width $d_x+2$,
    \begin{equation}
        \bar{\cN}(x)=\sum_{i=1}^Mw_i\sigma\left(\sum_{t=1}^N A_i[t]x[t]\right).
    \end{equation}
    By Lemma \ref{lemma:modified_rnn}, there exists an RNN $\cN:\bR^{d_x\times N}\rightarrow \bR$ of width $d_x+2+\beta(\sigma)$ such that
    \begin{equation}
        \sup_{x\in K^n}\left\|\bar{\cN}(x)-\cN(x)\right\|<\frac{\eps}{2}.
    \end{equation}
    Hence we have $\left\|f(x)-\cN(x)\right\|<\eps$.
\end{proof}

\reviewA{\subsection{Approximation of sequence-to-sequence function}}
\label{subsection:seq_to_seq}
Now, we consider an RNN $\cR$ as a function from sequence $x$ to sequence $y=\cR(x)$ defined by \eqref{eqn:rnn}.
Although the above results are remarkable in that the minimal width has an upper bound independent of the length of the sequence, it only approximates a part of the output sequence. 
Meanwhile, as the hidden states calculated in each RNN cell are connected closely for different times, fitting all the functions that can be independent of each other becomes a more challenging problem.
For example, the coefficient of $x[t-1]$ in $\cN(x)[t]$ equals the coefficient of $x[t]$ in $\cN(x)[t+1]$ if $\cN$ is an RNN defined as in the proof of Lemma \ref{lemma:linear_sum_into_rnn}.
This correlation originates from the fact that $x[t-1]$ and $x[t]$ arrive at $\cN(x)[t],\cN(x)[t+1]$ via the same intermediate process, 1-time step, and $N$ layers.

We sever the correlation between the coefficients of $x[t-1]$ and $x[t]$ by defining the \textit{time-enhanced recurrent cell} in Definition \ref{def:time_enhanced_layer} and proceed similarly as in the previous subsection till the Lemma \ref{lemma:mlp_in_trnn}.
\begin{definition}
    \label{def:time_enhanced_layer}
    \textit{Time-enhanced recurrent cell}, or layer, is a process that maps sequence $x=\left(x[t]\right)_{t\in\bN}\in\bR^{d_s\times\bN}$ to sequence $y=\left(y[t]\right)_{t\in\bN}\in\bR^{d_s\times\bN}$ via
    \begin{equation}
        \label{eqn:time_enhanced_layer}
        y[t+1]\coloneqq\cR(x)[t+1]=\sigma\left(A[t+1]\cR(x)[t]+B[t+1]x[t+1]+\theta[t+1]\right)
    \end{equation}
    where $\sigma$ is an activation function, $A[t]$, $B[t]\in\bR^{d_s\times d_s}$ are weight matrices and $\theta[t]\in\bR^{d_s}$ is the bias given for each time step $t$.
\end{definition}
Like RNN, \textit{time-enhanced RNN} indicates a composition of the form \eqref{eqn:rnn} with time-enhanced recurrent cells instead of RNN cells, and we denote it as TRNN.
The \textit{modified TRNN} indicates a TRNN whose activation functions in some cell act on only part of the components.
\textit{Time-enhanced BRNN}, denoted as \textit{TBRNN}, indicates a BRNN whose recurrent layers in each direction are replaced by time-enhanced layers.
A \textit{modified TBRNN} indicates a TBRNN whose activation function is modified to act on only part of the components.
With the proof of Lemma \ref{lemma:modified_rnn} using $\bar{A}[t]$, $\bar{B}[t]$ instead of $\bar{A}$, $\bar{B}$, a TRNN can approximate a modified TRNN.

The following lemma shows that the modified TRNN successfully eliminates the correlation between outputs.
See the appendix for the complete proof.
\begin{lemma}
    \label{lemma:linear_sum_into_trnn}
    Suppose $A_j[t]\in\bR^{1\times d_x}$ are the given matrices for $1\le t\le N$, $1\le j\le t$.
    Then there exists a modified TRNN $\tilde{\cN}:\bR^{d_x\times N}\rightarrow\bR^{(d_x+1)\times N}$ of width $d_x+1$ such that
    \begin{equation}
        \label{eqn:linear_sum_into_trnn}
        \tilde{\cN}(x)[t]=\begin{bmatrix}
        x[t]\\
        \sigma\left(\sum_{j=1}^tA_j[t]x[j]\right)\end{bmatrix},
    \end{equation}
    for all $t=1,2,\ldots,N$.
\end{lemma}
\begin{proof}
    [Sketch of proof]
    The detailed proof is available in Appendix \ref{appendix:linear_sum_into_trnn}.
    \review{The main idea is to use $b_m[t]\in\bR^{1\times d_x}$ instead of $b_m\in\bR^{1\times d_x}$ in the proof of Lemma \ref{lemma:linear_sum_into_rnn} with a modified TRNN $\tilde{\cR}$ of the form
    \begin{equation}
        \label{eqn:special_trnn_in_body}
        \tilde{\cR}(x)[t+1]\coloneqq \sigma_I\left(\begin{bmatrix}O_{d_x\times d_x}&\\&1\end{bmatrix}\tilde{\cR}(x)[t]+\begin{bmatrix}I_{d_x\times d_x}&0\\b_m[t+1]&1\end{bmatrix}x[t+1]\right).
    \end{equation}
    }
    \review{The time-index coefficient $b_m[t]$ separates the calculation along the same intermediate process mentioned at the beginning of this subsection.
    For instance, a process from time $t_1$ to $t_1+1$ results differently from $t_2$ to $t_2+1$ for $b_m[t_1]\ne b_m[t_2]$, even in the same layer.
    }
    As the coefficient matrices at each time $[t]$ after $N$ layers are full rank, we can find $b_m[t]$ implementing the required linear combination for each time.
\end{proof}

Recall the proof of Lemma \ref{lemma:mlp_in_rnn}.
An additional width serves as a buffer to implement and accumulate linear sum in a node in an MLP.
Similarly, we proceed with Lemma \ref{lemma:linear_sum_into_trnn} instead of Lemma \ref{lemma:linear_sum_into_rnn} to conclude that there exists a modified TRNN $\cN$ of width $d_x+2$ such that each $\cN[t]$ reproduces an MLP approximating $f[t]$.
\begin{lemma}
    \label{lemma:mlp_in_trnn}
    Suppose \reviewB{$w_i[t]\in\bR$}, $A_{i,j}[t]\in\bR^{1\times d_x}$ are \reviewA{the} given matrices for $1\le t\le N$, $1\le j\le t$, $1\le i\le M$.
    Then, there exists a modified TRNN $\tilde{\cN}:\bR^{d_x\times N}\rightarrow\bR^{1\times N}$ of width $d_x+2$ such that
    \begin{equation}
        \label{eqn:mlp_in_trnn}
        \tilde{\cN}(x)[t]=\sum_{i=1}^M \reviewB{w_i[t]}\sigma\left(\sum_{j=1}^t A_{i,j}[t]x[j]\right)
    \end{equation}
\end{lemma}
\begin{proof}
    We omit the detailed proof because it is almost the same as the proof of Lemma \ref{lemma:mlp_in_rnn}.
    \reviewB{The only difference is to use $b_m[t]$} instead of $b_m$ to construct $A_{i,j}\review{[t]}$ and $w_i[t]$ as in the proof of Lemma \ref{lemma:linear_sum_into_trnn}.
\end{proof}
This implies that the modified TRNN can approximate any past-dependent sequence-to-sequence function.

Finally, we connect the TRNN and RNN.
Although it is unclear whether a modified RNN can approximate an arbitrary modified TRNN, there exists a modified RNN that approximates the specific one described in Lemma \ref{lemma:linear_sum_into_trnn}.
\begin{lemma}
    \label{lemma:rnn_approximate_trnn}
    Let $\tilde{\cN}$ be a given modified TRNN that computes \eqref{eqn:linear_sum_into_trnn} with width $d_x+1$ and $K\subset\bR^{d_x}$ be a compact set.
    Then, for any $\eps>0$ there exists a modified RNN $\cN$ of width $d_x+2+\gamma(\sigma)$ such that
    \begin{equation}
        \sup_{x\in K^N}\left\|\tilde{\cN}(x)-\cN(x)\right\|<\eps,
    \end{equation}
    where $\gamma(\ReLU)=0$, $\gamma(\sigma)=1$ for non-degenerating activation $\sigma$.
    As a corollary, there exists a modified RNN $\cN$ of width $d_x+3+\gamma\left(\sigma\right)$ approximating $\tilde{\cN}$ in Lemma \ref{lemma:mlp_in_trnn}.
\end{lemma}
\begin{proof}
    [Sketch of proof]
    Detailed proof is available in Appendix \ref{appendix:rnn_approximate_trnn}.
    \reviewB{Note that a function $\tilde{\cN}$ in \eqref{eqn:linear_sum_into_trnn} is constructed by modified TRNN cells of the form in \eqref{eqn:special_trnn_in_body}.
    Since $b_m[t]$ is the only time-dependent coefficient in \eqref{eqn:special_trnn_in_body}, it is enough to show that an RNN cell can approximate a function $X\rightarrow\begin{bmatrix}I_{d_x}&O_{d_x\times 1}\\ b_m[t]&\review{1}\end{bmatrix}X$ for $X=\begin{bmatrix}x[t]\\0\end{bmatrix}\in\bR^{d_x+1}$.
    In particular, approximating a map $x[t]\mapsto b_m[t]x[t]$ is what we need.}
    
    \reviewB{
        While an MLP can approximate map $x[t]\mapsto b_m[t]x[t]$ on a compact set $K$, it is hard to use token-wise MLP since we need different outputs for the same vector at other time steps.
        For instance, for a $v\in K$ and a sequence $x=\left(x[1],x[2]\right)=\left(v,v\right)\in K^2$, a token-wise MLP cannot simultaneously approximate two outputs, $b_m[1]v$ and $b_m[2]v$.
        Still, it is not the case when the domain $K_1$ of $x[1]$ and the domain $K_2$ of $x[2]$ are disjoint compact sets, and there is an MLP that approximates $v\mapsto b_m[1]v$ for $v\in K_1$ and $v\mapsto b_m[2]v$ for $v\in K_2$.
        From this point of view, we use an RNN and a token-wise MLP, to embed each vector of the input sequence into disjoint compact sets and to approximate the output $b_m[t]x[t]$ on each compact set, respectively.
        
        Now we present how to construct disjoint sets and the token-wise MLP.
    }

    Without loss of generality, we can assume $K\subset\left[0,\frac{1}{2}\right]^{d_x}$ and construct the first cell as the output at time $t$ to be $x[t]+t\mathbf{1}_{d_x}$.
    As $N$ compact sets $K+\mathbf{1}_{d_x}$, $K+2\mathbf{1}_{d_x},\ldots, K+N\mathbf{1}_{d_x}$ are disjoint, $t$ and $x[t]=y-t$ are uniquely determined for each $y\in K+t\mathbf{1}_{d_x}$.
    Thus, for any given $b[t]\in\bR^{1\times d_x}$, there exists an MLP of width $d_x+1+\gamma\left(\sigma\right)$ approximating $y=x[t]+t\mathbf{1}_{d_x}\mapsto b[t]x[t]$ on $\bigsqcup_{t}\left(K+t\mathbf{1}_{d_x}\right)$ as a function from $\bR^{d_x}\rightarrow\bR$ \citep{hanin2017approximating, Kidger_2020}.
    Indeed, we need to approximate $x[t]+t\mathbf{1}_{d_x}\mapsto \begin{bmatrix}x[t]+t\mathbf{1}_{d_x}\\ b[t]x[t]\end{bmatrix}$ as a function from $\bR^{d_x}$ to $\bR^{d_x+1}$.
    Fortunately, the first $d_x$ components preserve the original input data in \reviewAB{the proofs of \textbf{Proposition 2} in \cite{hanin2017approximating}, and} \textbf{Proposition 4.2(Register Model)} in \cite{Kidger_2020}.
    Thus, an MLP of width $d_x+1+\reviewB{\gamma\left(\sigma\right)}$ approximates $b[t]x[t]$ while keeping the $x+t\mathbf{1}_{d_x}$ in the first $d_x$ components, and the MLP is common across overall $t$ as inputs at different times $t_1$ and $t_2$ are already embedded in disjoint sets $K+t_1\mathbf{1}_{d_x}$ and $K+t_2\mathbf{1}_{d_x}$.
    Note that a token-wise MLP is a special case of an RNN of the same width.
    Nonetheless, we need an additional width to keep the $(d_x+1)$th component approximating $b[t]x[t]$.
    With the token-wise MLP implemented by an RNN and additional buffer width, we construct a modified RNN of width $d_x+2+\gamma(\sigma)$ approximating the modified TRNN cell used in the proof of Lemma \ref{lemma:linear_sum_into_trnn}.
\end{proof}

\reviewA{\subsection{Proof of Theorem \ref{thm:universal_seq_to_seq}}}
Summarizing all the results, we have the universality of a deep RNN in a continuous function space.
\begin{proof}
    [Proof of Theorem \ref{thm:universal_seq_to_seq}]
    As mentioned in Remark \ref{remark:general_dy}, we can set $d_y=1$ for notational convenience.
    By Lemma \ref{lemma:mlp_in_trnn}, there exists a modified TRNN $\tilde{\cN}$ of width $d_x+2$ such that
    \begin{equation}
        \sup_{x\in K^n}\left\| f(x)-\tilde{\cN}(x)\right\|<\frac{\eps}{3}.
    \end{equation}
    As $\tilde{\cN}$ is a composition of modified TRNN cells of width $d_x+2$ satisfying \eqref{eqn:linear_sum_into_trnn}, there exists a modified RNN $\bar{\cN}$ of width $d_x+3+\gamma(\sigma)$ such that
    \begin{equation}
        \sup_{x\in K^n}\left\|\tilde{\cN}(x)-\bar{\cN}(x)\right\|<\frac{\eps}{3}.
    \end{equation}
    Then, by Lemma \ref{lemma:modified_rnn}, there exists an RNN $\cN$ of width $d_x+3+\gamma(\sigma)+\beta(\sigma)=d_x+\review{4}+\alpha(\sigma)$ such that
    \begin{equation}
        \sup_{x\in K^n}\left\|\bar{\cN}(x)-\cN(x)\right\|<\frac{\eps}{3}.
    \end{equation}
    The triangle inequality yields 
    \begin{equation}
        \sup_{x\in K^n}\left\|f(x)-\cN(x)\right\|<\eps.
    \end{equation}
\end{proof}

\begin{remark}
    The number of additional widths $\alpha(\sigma)=\beta(\sigma)+\gamma(\sigma)$ depends on the condition of the activation function $\sigma$.
    Here, $\gamma(\sigma)$ is required to find the token-wise MLP that approximates embedding from $\bR^{d_x}$ to $\bR^{d_x+1}$.
    If further studies determine a tighter upper bound of the minimum width of an MLP to have the universal property in a continuous function space, we can reduce or even remove $\alpha(\sigma)$ according to the result.
\end{remark}
There is still a wide gap between the lower bound $d_x$ and upper bound $d_x+d_y+\review{4}+\alpha(\sigma)$ of the minimum width, and hence, we expect to be able to achieve universality with a narrower width.
For example, if $N=1$, an RNN is simply an MLP, and the RNN has universality without a node required to compute the effect of $t$.
Therefore, apart from the result of the minimum width of an MLP, further studies are required to determine whether $\gamma$ is essential for the case of $N\ge2$.


 \section{Universal Approximation for Stack RNN in $L^p$ Space}
\label{section:universal_lp}
This section introduces the universal approximation theorem of a deep RNN in $L^p$ function space for $1\le p<\infty$.
\begin{theorem}[Universal approximation theorem of deep RNN 3]
    \label{thm:universal_lp}
    Let $f:\bR^{d_x\times N}\rightarrow\bR^{d_y\times N}$ be a past-dependent sequence-to-sequence function in $L^p\left(\bR^{d_x\times N},\bR^{d_y\times N}\right)$ for $1\le p
    <\infty$, and $\sigma$ be a non-degenerate activation function with the non-degenerating point $z_0$.  
    Then, for any $\eps>0$ and compact subset $K\subset\bR^{d_x}$, there exists a deep RNN $\cN$ of width $\max\left\{d_x+1, d_y\right\}+1$ satisfying
    \begin{equation}
        \sup_{1\le t\le N}\left\|f(x)[t]-\cN(x)[t]\right\|_{L^p\left(K\right)}<\eps.
    \end{equation}
    Moreover, if the activation $\sigma$ is ReLU, there exists a deep RNN $\cN$ of width $\max\left\{d_x+1, d_y\right\}$ satisfying
    \begin{equation}
        \sup_{1\le t\le N}\left\|f(x)[t]-\cN(x)[t]\right\|_{L^p\left(\bR^{d_x}\right)}<\eps.
    \end{equation}
\end{theorem}

Before beginning the proof of the theorem, we \reviewB{assume $K\subseteq \left[0,1\right]^{d_x}$ without loss of generality and} summarize the scheme that will be used in the proof.
\cite{Park_2021} constructed an MLP of width $\max\{d_x+1,d_y\}+\gamma(\sigma)$ approximating a given target function $f$, using the ``encoding scheme.''
More concretely, the MLP is separated into three parts: encoder, memorizer, and decoder.

First, the encoder part quantizes each component of the input and output into a finite set.
The authors use the quantization function $q_n:[0,1]\rightarrow \cC_n$
\begin{equation}
    \label{eqn:quantization}
    q_n(v)\coloneqq\max\left\{c\in\cC_n\bigm|c\le v\right\},
\end{equation}
where $\cC_n\coloneqq \left\{0,2^{-n},2\times2^{-n},\ldots,1-2^{-n}\right\}$.
Then, each quantized vector is encoded into a real number by concatenating its components through the encoder $\enc_M:[0,1]^{d_x}\rightarrow\cC_{d_xM}$
\begin{equation}
    \label{eqn:encoder}
        \enc_M(x)\coloneqq\sum_{i=1}^{d_x}q_M(x_i)2^{-(i-1)M}.
\end{equation}
For small $\delta_1>0$, the authors construct an MLP $\cN_{enc}:[0,1]^{d_x}\rightarrow\cC_{d_xM}$ of width $d_x+1+\gamma(\sigma)$ satisfying
\begin{equation}
    \label{eqn:mlp_enc}
    \left\|\enc_M(x)-\cN_{enc}(x)\right\|<\delta_1.
\end{equation}
Although the quantization causes a loss in input information, the $L^p$ norm neglects some loss in a sufficiently small domain.

After encoding the input $x$ to $\enc_M(x)$ with large $M$, authors use the information of $x$ in $\enc_M(x)$ to obtain the information of the target output $f(x)$.
More precisely, they define the memorizer $\mem_{M,M'}:\cC_{d_xM}\rightarrow\cC_{d_yM'}$ to map the encoded input $\enc_M(x)$ to the encoded output $\enc_{M'}(f(x))$ as
\begin{equation}
    \mem\left(\enc_M(x)\right)\coloneqq \left(\enc_{M'}\circ f\circ q_M\right)(x),
\end{equation}
assuming the quantized map $q_M$ acts on $x$ component-wise in the above equation.
Then, an MLP $\cN_{mem}$ of width $2+\gamma(\sigma)$ approximates $\mem$; that is, for any $\delta_2>0$, there exists $\cN_{mem}$ satisfying
\begin{equation}
    \sup_{x\in [0,1]^{d_x}}\left\|\mem(\enc(x))-\cN_{mem}(\enc(x))\right\|<\delta_2.
\end{equation}

Finally, the decoder reconstructs the original output vector from the encoded output vector by cutting off the concatenated components.
Owing to the preceding encoder and memorizer, it is enough to define only the value of the decoder on $\cC_{d_yM'}$.
Hence the decoder $\dec:\cC_{d_yM'}\rightarrow\cC_{M'}^{d_y}\coloneqq\left(\cC_{M'}\right)^{d_y}$ is determined by
\begin{equation}
    \dec_{M'}(v)\coloneqq\hat{v}\qquad\text{where}\qquad \left\{\hat{v}\right\}\coloneqq \enc_{M'}^{-1}(v)\cap \cC_{M'}^{d_y}.
\end{equation}
Indeed, for small $\delta_3>0$, \cite{Park_2021} construct an MLP $\cN_{dec}:\cC_{d_yM'}\rightarrow\cC_{M'}^{d_y}$ of width $d_y+\gamma(\sigma)$ so that
\begin{equation}
    \label{eqn:mlp_dec}
    \left\|\dec_{M'}(v)-\cN_{dec}(v)\right\|<\delta_3
\end{equation}
Although \eqref{eqn:mlp_enc} and \eqref{eqn:mlp_dec} are not equations but approximations when the activation is just non-degenerate, the composition $\cN=\cN_{dec}\circ \cN_{mem}\circ \cN_{enc}$ approximates a target $f$ with sufficiently large $M,M'$ and sufficiently small $\delta_1,\delta_2$.

Let us return to the proof of Theorem \ref{thm:universal_lp}.
We construct the encoder, memorizer, and decoder similarly.
As the encoder and decoder is independent of time $t$, we use a token-wise MLP and modified RNNs define the token-wise MLPs.
On the other hand, the memorizer must work differently according to the time $t$ owing to the multiple output functions.
Instead of implementing various memorizers, we separate their input and output domains at each time by translation.
Then, it is enough to define one memorizer on the disjoint union of domains.

\begin{proof}
    [Proof of Theorem \ref{thm:universal_lp}]
    We first combine the token-wise encoder and translation for the separation of the domains.
    \reviewB{Consider the token-wise encoder $\enc_M:\bR^{d_x\times N}\rightarrow\bR^{1\times N}$, and the following recurrent cell $\cR:\bR^{1\times N}\rightarrow\bR^{1\times N}$
    \begin{equation}
        \label{eqn:RNN_encoder}
        \cR(v)[t+1] = 2^{-\left(d_xM+1\right)}\cR(v)[t]+v[t+1]+1.
    \end{equation}
    Then the composition $\cR_{enc}=\cR\enc_M$ defines an encoder of sequence from $K^N$ to $\bR^{1\times N}$:
    \begin{equation}
        \cR_{enc}(x)[t]=\sum_{j=1}^t 2^{-(j-1)d_xM} + \sum_{j=1}^t\enc_M\left(x[j]\right)2^{-\left(j-1\right)\left(d_xM+1\right)},
    \end{equation}}
    where $x=\left(x[t]\right)_{t=1,\ldots,N}$ is a sequence in $K$.
    Note that the range $D$ of $\cR_{enc}$ is a disjoint union of compact sets;
    \begin{equation}
        D=\bigsqcup_{t=1}^N \left\{ \cR_{enc}(x)[t] \mathbin{:}x\in K^N\right\}.
    \end{equation}
    Hence there exists a memorizer $\mem:\bR\rightarrow\bR$ satisfying
    \begin{equation}
        \mem(\cR_{enc}(x)\reviewB{[t]}) = \enc_{M'}\left( f\left(q_M(x)\right)[t]\right)
    \end{equation}
    for each $t=1,2,\ldots,N$. The token-wise decoder $\dec_{M'}$ is the last part of the proof.

    To complete the proof, we need an approximation of the token-wise encoder $\enc_M:\bR^{d_x}\rightarrow\bR$, \reviewB{a modified recurrent cell $\cR:\bR^{1\times N}\rightarrow\bR^{1\times N}$}, token-wise memorizer $\mem:\bR\rightarrow\bR$, and token-wise decoder $\dec_{M'}:\bR\rightarrow\bR^{d_y}$.
    Following \cite{Park_2021}, there exist MLPs of width $d_x+1+\gamma(\sigma)$, $2+\gamma(\sigma)$, and $d_y+\gamma(\sigma)$ that approximate $\enc_M$, $\mem$, and $\dec_{M'}$ respectively.
    \reviewB{Lemma \ref{lemma:modified_rnn} shows that $\cR$ is approximated by an RNN of width $1+\beta(\sigma)$.}
    Hence, an RNN of width $\max\left\{d_x+1+\gamma(\sigma), \reviewB{1+\beta(\sigma)}, 2+\gamma(\sigma),d_y+\gamma(\sigma)\right\}=\max\left\{d_x+1,d_y\right\}+\gamma(\sigma)$ approximates the target function $f$.

    \reviewB{In the case of ReLU activation, we can extend the domain $K^N$ to $\bR^{d_x\times N}$ as stated in \cite{park2020minimum}.
    Nonetheless, we present briefly how to deal with the subtle problem that the support of a network is not compact generally.
    
    We project each input $x[t]\in\bR^{d_x}$ by $P_{L,\delta}:\bR^{d_x}\rightarrow \bR^{d_x}$ defined by
    \begin{equation}
        P(x)=\begin{cases}
            x&x\in [-L,L], \\
            -\frac{L}{\delta}\left(x+L+\delta\right)&x\in[-L-\delta,-L], \\
            -\frac{L}{\delta}\left(x-L-\delta\right)&x\in[L,L+\delta], \\
            0&\text{otherwise}.
        \end{cases}
    \end{equation}
    Note that an MLP with width $d_x+1$ computes $P_{L,\delta}$.
    We will choose large $L$ to cover enough domain for $x\in\bR^{d_x\times N}$ and small $\delta$ to reduce the projection error.
    
    First, choose a large $L$ so that $\left\|f(x)[t]\right\|_{L^p\left(\bR^{d_x}\backslash[-L,L]^{d_x}\right)}<\frac{\eps}{3}$ for all $t=1,2,\ldots,N$.
    Second, construct an RNN $\cN=\cN_{dec}\circ\cN_{mem}\circ\cN_{enc}$ as above with $\cN(0)=0$, such that $\left\|f(x)[t]-\cN(x)[t]\right\|_{L^p\left([-L,L]^{d_x}\right)}<\frac{\eps}{3}$.
    We impose a condition $\cN(0)=0$ to bound error outside $[-L,L]^{d_x}$, which is possible because we may set $f(0)[t]=0$ for all $t$ under the $L^p$ norm.
    Finally, set $\delta$ so small that $\left\|f(x)[t]\right\|_{L^p\left([-L-\delta,L+\delta]^{d_x}\backslash [-L,L]^{d_x}\right)}<\frac{\eps}{6}$ and $\left\|\cN(x)[t]\right\|_{L^p\left([-L-\delta,L+\delta]^{d_x}\backslash [-L,L]^{d_x}\right)}<\frac{\eps}{6}$.
    After inserting the projection $P_{L,\delta}$ before the encoder $\cN_{enc}$ and defining new RNN $\cN'=\cN\circ P_{L,\delta}$, we have
    \begin{align}
        &\left\|f(x)[t]-\cN'(x)[t]\right\|_{L^p\left(\bR^{d_x}\right)} \\
        & \le \left\|f(x)[t]-\cN(x)[t]\right\|_{L^p\left([-L,L]^{d_x}\right)}
        +\left\|f(x)[t]-\cN'(x)[t]\right\|_{L^p\left(\bR^{d_x}\backslash[-L,L]^{d_x}\right)}\\
        &\le \frac{\eps}{3} + \left\| f(x)[t]-\cN(x)[t]\right\|_{L^p\left([-L-\delta,L+\delta]^{d_x}\backslash [-L,L]^{d_x}\right)}\\
        &\qquad +\left\| f(x)[t]-\cN(x)[t]\right\|_{L^p\left(\bR^{d_x}\backslash [-L-\delta,L+\delta]^{d_x}\right)}\\
        & \le \frac{\eps}{3} + \frac{\eps}{3} + \left\| f(x)[t]\right\|_{L^p\left(\bR^{d_x}\backslash[-L,L]^{d_x}\right)}\\
        &\le \eps
    \end{align}}
\end{proof}


\section{Variants of RNN}
\label{section:variants_of_rnn}
This section describes the universal property of some variants of RNN, particularly LSTM, GRU, or BRNN.
LSTM and GRU are proposed to solve the long-term dependency problem.
As an RNN has difficulty calculating and updating its parameters for long sequential data, LSTM and GRU take advantage of additional structures in their cells.
We prove that they have the same universal property as the original RNN.
On the other hand, a BRNN is proposed to overcome the past dependency of an RNN.
BRNN consists of two RNN cells, one of which works in reverse order.
We prove the universal approximation theorem of a BRNN with the target class of any sequence-to-sequence function.

The universal property of an LSTM originates from the universality of an RNN.
Mathematically LSTM $\cR_{LSTM}$ indicates a process that computes two outputs, $h$ and $c$, defined by \eqref{eqn:lstm}.
As an LSTM can reproduce an RNN with the same width, we have the following corollary:
\begin{corollary}[Universal approximation theorem of deep LSTM]
    \label{cor:universal_lstm}
    Let $f:\bR^{d_x\times N}\rightarrow\bR^{d_y\times N}$ be a continuous past-dependent sequence-to-sequence function.
    Then, for any $\eps>0$ and compact subset $K\subset\bR^{d_x}$, there exists a deep LSTM $\cN_{LSTM}$, of width $d_x+d_y+3$, such that
    \begin{equation}
        \sup_{x\in K^N}\sup_{1\le t\le N}\left\|f(x)[t]-\cN_{LSTM}(x)[t]\right\|<\eps.
    \end{equation}
    \reviewB{If $f\in L^p\left(\bR^{d_x\times N},\bR^{d_y\times N}\right)$, there exists a deep LSTM $\cN_{LSTM}$, of width $\max\left\{d_x+1,d_y\right\}+1$, such that
    \begin{equation}
        \sup_{1\le t\le N}\left\|f(x)[t]-\cN_{LSTM}(x)[t]\right\|_{L^p\left(K^N\right)}<\eps.
    \end{equation}}
\end{corollary}
\begin{proof}
    \reviewB{We set all parameters but $W_g$, $U_g$, $b_g$, and $b_f$ as zeros, and then \eqref{eqn:lstm} is simplified as
    \begin{equation}
        \label{eqn:simple_lstm}
        \begin{aligned}
            c[t+1]&=\sigmoid(b_f)\odot c[t]+\frac{1}{2}\tanh\left(U_gh[t]+W_gx[t+1]+b_g\right),\\
            h[t+1]&=\frac{1}{2}\tanh\left(c[t+1]\right).
        \end{aligned}
    \end{equation}
    For any $\eps>0$, $b_f$ with sufficiently large negative components yields
    \begin{equation}
        \left\|h[t+1]-\frac{1}{2}\tanh\left(\frac{1}{2}\tanh\left(U_gh[t]+W_gx[t+1]+b_g\right)\right)\right\|<\eps.
    \end{equation}}
    Thus, an LSTM reproduces an RNN whose activation function is $\left(\frac{1}{2}\tanh\right)\circ\left(\frac{1}{2}\tanh\right)$ without any additional width in its hidden states.
    In other words, an LSTM of width $d$ approximates an RNN of width $d$ equipped with the activation function $\left(\frac{1}{2}\tanh\right)\circ\left(\frac{1}{2}\tanh\right)$.
\end{proof}

The universality of GRU is proved similarly.
\begin{corollary}[Universal approximation theorem of deep GRU]
    \label{cor:universal_gru}
    Let $f:\bR^{d_x\times N}\rightarrow\bR^{d_y\times N}$ be a continuous past-dependent sequence-to-sequence function.
    Then, for any $\eps>0$ and compact subset $K\subset\bR^{d_x}$, there exists a deep GRU $\cN_{GRU}$, of width $d_x+d_y+3$, such that
    \begin{equation}
        \sup_{x\in K^N}\sup_{1\le t\le N}\left\|f(x)[t]-\cN_{GRU}(x)[t]\right\|<\eps.
    \end{equation}
    \reviewB{If $f\in L^p\left(\bR^{d_x\times N},\bR^{d_y\times N}\right)$, there exists a deep GRU $\cN_{GRU}$, of width $\max\left\{d_x+1,d_y\right\}+1$, such that
    \begin{equation}
        \sup_{1\le t\le N}\left\|f(x)[t]-\cN_{GRU}(x)[t]\right\|_{L^p\left(K^N\right)}<\eps.
    \end{equation}}
\end{corollary}
\begin{proof}
    Setting only $W_h$, $U_h$, $b_h$, and $b_z$ as non-zero, the GRU is simplified as
    \begin{equation}
        \label{eqn:simple_gru}
        h[t+1]=\left(1-\sigmoid\left(b_z\right)\right)h[t]+\sigmoid\left(b_z\right)\tanh\left(W_hx[t+1]+\frac{1}{2}U_hh[t]+b_h\right).
    \end{equation}
    For any $\eps>0$, a sufficiently large $b_z$ yields
    \begin{equation}
        \left\|h[t+1]-\tanh\left(W_hx[t+1]+\frac{1}{2}U_hh[t]+b_h\right)\right\|<\eps.
    \end{equation}
    Hence, we attain the corollary.
\end{proof}

\begin{remark}
    We refer to the width as the maximum of hidden states.
    However, the definition is somewhat inappropriate, as LSTM and GRU cells have multiple hidden states; hence, there are several times more components than an RNN with the same width.
    Thus we expect that they have better approximation power or have a smaller minimum width for universality than an RNN.
    Nevertheless, we retain the theoretical proof as future work to identify whether they have different abilities in approximation or examine why they exhibit different performances in practical applications.
\end{remark}

Now, let us focus on the universality of a BRNN.
Recall that a stack of modified recurrent cells $\cN$ construct a linear combination of the previous input components $x[1:t]$ at each time,
\begin{equation}
    \label{eqn:forward_linear_sum}
    \cN(x)[t]=\begin{bmatrix}x[t]\\ \sum_{j=1}^tA_j[t]x[j]\end{bmatrix}.
\end{equation}
Therefore, if we reverse the order of sequence and flow of the recurrent structure, a stack of reverse modified recurrent cells $\bar{\cN}$ constructs a linear combination of the subsequent input components $x[t:N]$ at each time,
\begin{equation}
    \label{eqn:backward_linear_sum}
    \bar{\cN}(x)[t]=\begin{bmatrix}x[t]\\ \sum_{j=t}^N B_j[t]x[j]\end{bmatrix}.
\end{equation}
From this point of view, we expect that a stacked BRNN successfully approximates an arbitrary sequence-to-sequence function beyond the past dependency.
As previously mentioned, we prove it in the following lemma.
\begin{lemma}
    \label{lemma:linear_sum_into_tbrnn}
    Suppose $A_j[t]\in\bR^{1\times d_x}$ are the given matrices for $1\le t\le N$, $1\le j\le N$.
    Then there exists a modified TBRNN $\tilde{\cN}:\bR^{d_x\times N}\rightarrow\bR^{(d_x+1)\times N}$ of width $d_x+1$ such that
    \begin{equation}
        \label{eqn:linear_sum_into_tbrnn}
        \tilde{\cN}(x)[t]=\begin{bmatrix} x[t]\\ \sigma\left(\sum_{j=1}^NA_j[t]x[j]\right)\end{bmatrix},
    \end{equation}
    for all $t=1,2,\ldots,N$.
\end{lemma}
\begin{proof}
    [Sketch of proof]
    The detailed proof is available in Appendix \ref{appendix:linear_sum_into_tbrnn}.
    We use modified TBRNN cells with either only a forward modified TRNN or a backward modified TRNN.
    The stacked forward modified TRNN cells compute $\sum_{j=1}^tA_j[t]x[j]$, and the stacked backward modified TRNN cells compute $\sum_{j=t+1}^NA_j[t]x[j]$.
\end{proof}

As in previous cases, we have the following theorem for a TBRNN.
The proof is almost the same as that of Lemma \ref{lemma:mlp_in_trnn} and \ref{lemma:mlp_in_rnn}.
\begin{lemma}
    \label{lemma:mlp_in_tbrnn}
    Suppose \reviewB{$w_i[t]\in\bR$}, $A_{i,j}[t]\in\bR^{1\times d_x}$ are the given matrices for $1\le t\le N$, $1\le j\le N$, $1\le i\le M$. Then there exists a modified TBRNN $\tilde{\cN}:\bR^{d_x\times N}\rightarrow\bR^{1\times N}$ of width $d_x+2$ such that
    \begin{equation}
        \tilde{\cN}(x)[t]=\sum_{i=1}^M\reviewB{w_i[t]}\sigma\left(\sum_{j=1}^NA_{i,j}[t]x[j]\right).
    \end{equation}
\end{lemma}
\begin{proof}
    First, construct a modified deep TBRNN $\cN_1:\bR^{d_x\times N}\rightarrow\bR^{(d_x+2)\times N}$ of width $d_x+2$ such that
    \begin{equation}
        \cN_1(x)[t]=\begin{bmatrix}x[t]\\ \sigma\left(\sum_{j=1}^NA_{1,j}[t]x[j]\right)\\0\end{bmatrix},
    \end{equation}
    as Lemma \ref{lemma:linear_sum_into_tbrnn}.
    The final component does not affect the first linear summation and remains zero.
    After $\cN_1$, use the $(d_x+2)$th component to obtain a stack of cells $\cN_2:\bR^{(d_x+2)\times N}\rightarrow\bR^{(d_x+2)\times N}$, which satisfies
    \begin{equation}
        \cN_2\cN_1(x)[t]=\begin{bmatrix}x[t]\\ \sigma\left(\sum_{j=1}^NA_{1,j}[t]x[j]\right)\\ \sigma\left(\sum_{j=1}^NA_{2,j}[t]x[j]\right)\end{bmatrix},
    \end{equation}
    and use a modified RNN cell $\cR$ to sum up the results and reset the last component:
    \begin{equation}
        \cR\cN_2\cN_1(x)[t] =\begin{bmatrix}x[t]\\ w_1[t]\sigma\left(\sum_{j=1}^NA_{1,j}[t]x[j]\right)+w_2[t]\sigma\left(\sum_{j=1}^NA_{2,j}[t]x[j]\right)\\0\end{bmatrix}.
    \end{equation}
    As the $(d_x+2)$th component resets to zero, we use it to compute the third sum $w_3[t]\sigma\left(\sum A_{3,j}[t]x[j]\right)$ and repeat until we obtain the final network $\cN$ such that
    \begin{equation}
        \cN(x)[t] = \begin{bmatrix}x[t]\\\sum_{i=1}^Mw_i[t]\sigma\left(\sum_{t=1}^NA_{i,j}[t]x[j]\right)\\0\end{bmatrix}.
    \end{equation}
\end{proof}

The following lemma fills the gap between a modified TBRNN and a modified BRNN.
\begin{lemma}
    \label{lemma:brnn_approximate_tbrnn}
    Let $\tilde{\cN}$ be a modified TBRNN that computes \eqref{eqn:linear_sum_into_tbrnn} and $K\subset\bR^{d_x}$ be a compact set.
    Then for any $\eps>0$ there exists a modified BRNN $\bar{\cN}$ of width $d_x+2+\gamma(\sigma)$ such that
    \begin{equation}
        \sup_{x\in K^N}\left\|\tilde{\cN}(x)-\bar{\cN}(x)\right\|<\eps,
    \end{equation}
    where $\gamma(\ReLU)=0$, $\gamma(\sigma)=1$ for non-degenerating activation $\sigma$.
    \\
    Moreover, there exists a BRNN $\cN$ of width $d_x+\review{3}+\alpha(\sigma)$ such that
    \begin{equation}
        \sup_{x\in K^N}\left\|\tilde{\cN}(x)-\cN(x)\right\|<\eps,
    \end{equation}
    \review{where $\alpha\left(\sigma\right)=\begin{cases}
        0&\text{$\sigma$ is ReLU or a non-degenerating function with $\sigma\left(z_0\right)=0$.}\\
        1&\text{$\sigma$ is a non-degenerating function with $\sigma\left(z_0\right)\ne0$}.
    \end{cases}$}
\end{lemma}
\begin{proof}
    We omit these details because we only need to construct a modified RNN that approximates \eqref{eqn:forward_linear_sum} and \eqref{eqn:backward_linear_sum} using Lemma \ref{lemma:rnn_approximate_trnn}.
    As only the forward or backward modified RNN cell is used in the proof of Lemma \ref{lemma:linear_sum_into_tbrnn}, it is enough for the modified BRNN to approximate either the forward or backward modified TRNN.
    Thus, it follows from Lemma \ref{lemma:rnn_approximate_trnn}.
    Lemma \ref{lemma:modified_rnn} provides the second part of this theorem.
\end{proof}

Finally, we obtain the universal approximation theorem of the BRNN from the previous results.
\begin{theorem}[Universal approximation theorem of deep BRNN 1]
    \label{thm:universal_brnn}
    Let $f:\bR^{d_x\times N}\rightarrow\bR^{d_y\times N}$ be a continuous sequence to \review{sequence} function and $\sigma$ be a non-degenerate activation function. Then for any $\eps>0$ and compact subset $K\subset\bR^{d_x}$, there exists a deep BRNN $\cN$ of width $d_x+d_y+\review{3}+\alpha(\sigma)$, such that
    \begin{equation}
        \sup_{x\in K^N}\sup_{1\le t\le N}\left\|f(x)[t]-\cN(x)[t]\right\|<\eps,
    \end{equation}
    \review{where $\alpha\left(\sigma\right)=\begin{cases}
        0&\text{$\sigma$ is ReLU or a non-degenerating function with $\sigma\left(z_0\right)=0$.}\\
        1&\text{$\sigma$ is a non-degenerating function with $\sigma\left(z_0\right)\ne0$}.
    \end{cases}$}
\end{theorem}
\begin{proof}
    As in the proof of Theorem \ref{thm:universal_seq_to_seq}, we set $d_y=1$ for notational convenience.
    According Lemma \ref{lemma:mlp_in_tbrnn}, there exists a modified TBRNN $\tilde{\cN}$ of width $d_x+2$ such that
    \begin{equation}
        \sup_{x\in K^n}\left\|f(x)-\tilde{\cN}(x)\right\|<\frac{\eps}{2}.
    \end{equation}
    Lemma \ref{lemma:brnn_approximate_tbrnn} implies that there exists a BRNN of width $d_x+\review{4}+\alpha(\sigma)$ such that
    \begin{equation}
        \sup_{x\in K^n}\left\|\tilde{\cN}(x)-\cN(x)\right\|<\frac{\eps}{2}.
    \end{equation}
    The triangle inequality leads to
    \begin{equation}
        \sup_{x\in K^n}\left\|f(x)-\cN(x)\right\|<\eps.
    \end{equation}
\end{proof}

In the $L^p$ space, we have similar results with RNNs much more straightforward than in continuous function spaces as the only encoder needs bidirectional flow.
\reviewB{
\begin{theorem}[Universal approximation theorem of deep BRNN 2]
    \label{thm:universal_brnn_lp}
    Let $f:\bR^{d_x\times N}\rightarrow\bR^{d_y\times N}$ be a sequence-to-sequence function in $L^p\left(\bR^{d_x\times N},\bR^{d_y\times N}\right)$ for $1\le p
    <\infty$, and $\sigma$ be a non-degenerate activation function with the non-degenerating point $z_0$.
    Then, for any $\eps>0$ and compact subset $K\subset\bR^{d_x}$, there exists a deep BRNN $\cN$ of width $\max\left\{d_x+1, d_y\right\}+1$ satisfying
    \begin{equation}
        \sup_{1\le t\le N}\left\|f(x)[t]-\cN(x)[t]\right\|_{L^p\left(K\right)}<\eps.
    \end{equation}
    Moreover, if the activation $\sigma$ is ReLU, there exists a deep BRNN $\cN$ of width $\max\left\{d_x+1, d_y\right\}$ satisfying
    \begin{equation}
        \sup_{1\le t\le N}\left\|f(x)[t]-\cN(x)[t]\right\|_{L^p\left(\bR^{d_x}\right)}<\eps.
    \end{equation}
\end{theorem}
\begin{proof}
    Recall that in the RNN case, $x[1:t]\in K^N$ is encoded as a concatenation of a decimal representation of each $x[1],x[2],\ldots,x[t]$, cutting it off at a certain number of digits.
    The backward process carries the same work, and then the encoded results will be connected.
    Since the cells constructing the encoders are the same as equation \eqref{eqn:RNN_encoder}, we omit the further step.
\end{proof}}


\section{Discussion}
\label{section:discussion}
We proved the universal approximation theorem and calculated the upper bound of the minimum width of an RNN, an LSTM, a GRU, and a BRNN.
In this section, we illustrate how our results support the performance of a recurrent network.

We show that an RNN needs a width of at most $d_x+d_y+4$ to approximate a function from a sequence of $d_x$-dimensional vectors to a sequence of $d_y$-dimensional vectors.
The upper bound of the minimum width of the network depends only on the input and output dimensions, regardless of the length of the sequence.
The independence of the sequence length indicates that the recurrent structure is much more effective in learning a function on sequential inputs.
To approximate a function defined on a long sequence, a network with a feed-forward structure requires a wide width proportional to the length of the sequence.
For example, an MLP should have a wider width than $Nd_x$ if it approximates a function $f:\bR^{d_x\times N}\rightarrow\bR$ defined on a sequence \citep{Johnson_2019}.
However, with the recurrent structure, it is possible to approximate via a narrow network of width $d_x+1$ regardless of the length, because the minimum width is independent of the length $N$.
This suggests that the recurrent structure, which transfers information between different time steps in the same layer, is crucial for success with sequential data.

From a practical point of view, this fact further implies that there is no need to limit the length of the time steps that affect dynamics to learn the internal dynamics between sequential data.
For instance, suppose that a pair of long sequential data $(x[t])$ and $(y[t])$ have an unknown relation $y[t]=f\left(x[t-p],x[t-p+1],\ldots,x[t]\right)$.
Even without prior knowledge of $f$ and $p$, a deep RNN learns the relation if we train the network with inputs $x[1:t]$ and outputs $y[t]$.
The MLP cannot reproduce the result because the required width increases proportionally to $p$, which is an unknown factor.
The difference between these networks theoretically supports that recurrent networks are appropriate when dealing with sequential data whose underlying dynamics are unknown in the real world.

\section{Conclusion}
\label{section:conclusion}
In this study, we investigated the universality and upper bound of the minimum width of deep RNNs.
The upper bound of the minimum width serves as a theoretical basis for the effectiveness of deep RNNs, especially when underlying dynamics of the data are unknown.

Our methodology enables various follow-up studies, as it connects an MLP and a deep RNN.
For example, the framework disentangles the time dependency of output sequence of an RNN.
This makes it feasible to investigate a trade-off between width and depth in the representation ability or error bounds of the deep RNN, which has not been studied because of the entangled flow with time and depth.
In addition, we separated the required width into three parts: one maintains inputs and results, another resolves the time dependency, and the third modifies the activation.
Assuming some underlying dynamics in the output sequence, such as an open dynamical system, we expect to reduce the required minimum width on each part because there is a natural dependency between the outputs, and the inputs are embedded in a specific way by the dynamics.

However, as LSTMs and GRUs have multiple hidden states in the cell process, they may have a smaller minimum width than the RNN.
By constructing an LSTM and a GRU to use the hidden states to save data and resolve the time dependency, we hope that our techniques demonstrated in the proof help analyze why these networks have a better result in practice and suffer less from long-term dependency.

\section*{Acknowledgement}
\review{This work was supported by the Challengeable Future Defense Technology Research and Development Program through ADD[No. 915020201], the NRF grant[2012R1A2C3010887] and  the MSIT/IITP[No. 2021-0-01343, Artificial Intelligence Graduate School Program(SNU)].}

\newpage
\appendix
\section{Notations}
\begin{center}
    \begin{table}
        \begin{tabular}{l l}
            \textbf{Symbol} & \textbf{Description} \\
            \toprule
            $\sigma$ & Activation function \\
            $\sigma_I$ & Modified activation function\\
            $x$ & Input sequence of vectors\\
            $y$ & Output sequence of vectors\\
            $t$ & Order of element in a sequence\\
            $a[t]_i$ & $i$-th component of $t$-element of a sequence $a$\\
            $d_x$ & Dimension of input vector\\
            $d_s$ & Dimension of hidden state in RNN, or width\\
            $d_y$ & Dimension of \review{output} vector\\
            $N$ & Length of the input and output sequence\\
            $K$ & A compact subset of $\bR^{d_x}$\\
            $O_{m,n}$ & Zero matrix of size $m\times n$\\
            $\mathbf{0}_k$ & Zero vector of size $k$, 
            $\mathbf{0}_k={\underbrace{\begin{bmatrix}0&0\cdots&0\end{bmatrix}}_k}^T$\\
            $\mathbf{1}_k$ & One vector of size $k$, $\mathbf{1}_k={\underbrace{\begin{bmatrix}1&1\cdots&1\end{bmatrix}}_k}^T$\\
            $I_k$ & Identity matrix of size $k\times k$\\
            RNN & Recurrent Neural Network defined in page \pageref{eqn:rnn}\\
            TRNN & Time-enhanced Recurrent Neural Network, defined by replacing\\
                    & recurrent cell with time-enhanced recurrent cell in RNN.             \\
                    & Time-enhenced cell and TRN are defined in page \pageref{def:time_enhanced_layer}\\
            BRNN & Bidirectional Recurrent Neural Network defined in page \pageref{eqn:brnn}\\
            TBRNN & Time-enhanced Bidirectional Recurrent Neural Network\\
                    & defined in page \pageref{def:time_enhanced_layer}
            \\
            \bottomrule
        \end{tabular}
    \end{table}
\end{center}


\section{Proof of the Lemma \ref{lemma:modified_rnn}}
\label{appendix:modified_rnn}
Without loss of generality, we may assume  $\bar{\cP}$ is an identity map and $I=\{1,2,\ldots,k\}$.
Let $\bar{\cR}(x)[t+1]=\sigma_I\left(\bar{A}\cR(x)[t]+\bar{B}x[t+1]+\bar{\theta}\right)$ be a given modified RNN cell, and $\cQ(x)[t]=\bar{Q}x[t]$ be a given token-wise linear projection map. 
We use notations $O_{m,n}$ and $\mathbf{1}_{m}$ to denote zero matrix in $\bR^{m\times n}$ and one vector in $\bR^m$ respectively. Sometimes we omit $O_{m,n}$ symbol in some block-diagonal matrices if the size of the zero matrix is clear.

\paragraph{Case 1: $\sigma(z_0)=0$}~
\\
Let $\cP$ be the identity map.
For $\delta>0$ define $\cR_1^\delta$ as
\begin{equation}
    \cR_1^\delta\left(x[t+1]\right)\coloneqq \sigma\left(\delta\bar{B}x[t+1]+\delta\bar{\theta}+z_0\mathbf{1}_d\right).
\end{equation}
Since $\sigma$ is non-degenerating at $z_0$ and $\sigma'$ is continuous at $z_0$, we have
\begin{equation}
    \cR_1^\delta\circ\cP(x)[t+1]=\delta\sigma'(z_0)\left(\bar{B}x[t+1]+\bar{\theta}\right)+o(\delta).
\end{equation}
Then construct a second cell to compute transition as
\begin{equation}
    \label{eqn:second_rnn_cell_1}
    \cR_2^\delta(x)[t+1]=\sigma\left(
    \tilde{A}\cR_2^\delta(x)[t]
    +\frac{1}{\sigma'(z_0)}\begin{bmatrix}\delta^{-1}I_{k}& \\ &I_{d-k}\end{bmatrix}x[t+1]
    +\begin{bmatrix}\mathbf{0}_k\\z_0\mathbf{1}_{d-k}\end{bmatrix}
    \right),
\end{equation}
where $\tilde{A}=\begin{bmatrix}I_k& \\ &\delta I_{d-k}\end{bmatrix}
\bar{A}
\begin{bmatrix}I_k& \\ &\frac{1}{\delta\sigma'(z_0)}I_{d-k}\end{bmatrix}$.
\\
After that, the first output of $\cR_2^\delta\cR_1^\delta\cP(x)$ becomes
\begin{align}
    \cR_2^\delta\cR_1^\delta\cP(x)[1]&
    =\sigma\left(\frac{1}{\sigma'(z_0)}\begin{bmatrix}\delta^{-1}I_k& \\ &I_{d-k}\end{bmatrix}\cR_1^\delta(x)[1]
    +\begin{bmatrix}\mathbf{0}_k\\z_0\mathbf{1}_{d-k}\end{bmatrix}\right)
    \\
    &=\sigma\left(\begin{bmatrix}\left(\bar{B}x[1]+\bar{\theta}\right)_{1:k}+\delta^{-1}o(\delta)\\
    \left(z_0\mathbf{1}_{d-k}+\delta(\bar{B}x[1]+\bar{\theta}\right)_{k+1:d}+o(\delta)\end{bmatrix}\right)
    \\
    &=\begin{bmatrix}\sigma\left(\bar{B}x[1]+\bar{\theta}\right)_{1:k}+o(1)\\
   \sigma'(z_0)\delta\left(\bar{B}x[1]+\bar{\theta}\right)_{k+1:d}+o(\delta)
    \end{bmatrix}
    \\
    &=\begin{bmatrix}\bar{\cR}(x)[1]_{1:k}+o(1)\\
    \sigma'(z_0)\delta\bar{\cR}(x)[1]_{k+1:d}+o(\delta)\end{bmatrix}.
\end{align}
Now use mathematical induction on time $t$ to compute $\cR_2^\delta\cR_1^\delta\cP(x)$ assuming
\begin{equation}
    \label{eqn:induction_hypothesis_1}
    \cR_2^\delta\cR_1^\delta\cP(x)[t]
    =\begin{bmatrix}\bar{\cR}(x)[t]_{1:k}+o(1)\\
    \sigma'(z_0)\delta\bar{\cR}(x)[t]_{k+1:d}+o(\delta)\end{bmatrix}.
\end{equation}
From a direct calculation, we attain
\begin{align}
    &\frac{1}{\sigma'(z_0)}\begin{bmatrix}\delta^{-1}I_k& \\ &I_{d-k}\end{bmatrix}\cR_1^\delta\cP(x)[t+1]
    +\begin{bmatrix}\mathbf{0}_k\\z_0\mathbf{1}_{d-k}\end{bmatrix}
    \\
    &=\frac{1}{\sigma'(z_0)}\begin{bmatrix}\delta^{-1}I_k& \\ &I_{d-k}\end{bmatrix}\left(\delta\sigma'(z_0)\left(\bar{B}x[t+1]+\bar{\theta}\right)+o(\delta)\right)
    +\begin{bmatrix}
        \mathbf{0}_k\\z_0\mathbf{1}_{d-k}
    \end{bmatrix}
    \\
    &=\begin{bmatrix}
        \bar{B}x[t+1]_{1:k}+\bar{\theta}_{1:k}+\delta^{-1}o(\delta)\\z_0\mathbf{1}_{d-k}+\delta\left(\bar{B}x[t+1]+\bar{\theta}\right)_{k+1:d}+o(\delta)
    \end{bmatrix},
\end{align}
and
\begin{align}
    &\tilde{A}\cR_2^\delta\cR_1^\delta\cP(x)[t]
    \\
    &=
    \begin{bmatrix}
        I_k&\\
        &\delta I_{d-k}
    \end{bmatrix}
    \bar{A}
    \begin{bmatrix}
        I_k& \\
        &\frac{1}{\delta\sigma'(z_0)}I_{d-k}
    \end{bmatrix}
    \begin{bmatrix}
        \bar{\cR}(x)[t]_{1:k}+o(1)\\
        \sigma'(z_0)\delta\bar{\cR}(x)[t]_{k+1:d}+o(\delta)
    \end{bmatrix}
    \\
    &=\begin{bmatrix}I_k& \\ &\delta I_{d-k}\end{bmatrix}
    \bar{A}
    \begin{bmatrix}\bar{\cR}(x)[t]_{1:k} +o(1)\\ \bar{\cR}(x)[t]_{k+1:d}+o(1)\end{bmatrix}
    \\
    &=\begin{bmatrix} \left(\bar{A}\bar{\cR}(x)[t]\right)_{1:k}+o(1) \\ \delta\left(\bar{A}\bar{\cR}(x)[t]\right)_{k+1:d}+o(\delta)\end{bmatrix}.
\end{align}
With the sum of above two results, we obtain the induction hypothesis \eqref{eqn:induction_hypothesis_1} for $t+1$,
\begin{align}
    &\cR_2^\delta\cR_1^\delta\cP(x)[t+1]
    \\
    &=\sigma\left(\tilde{A}\cR_2^\delta\cR_1^\delta\cP(x)[t]+\frac{1}{\sigma'(z_0)}\begin{bmatrix}\delta^{-1}I_k& \\ &I_{d-k}\end{bmatrix}\cR_1^\delta\cP(x)[t+1]
    +\begin{bmatrix}\mathbf{0}_k\\z_0\mathbf{1}_{d-k}\end{bmatrix}\right)
    \\
    &=\sigma\begin{bmatrix}
        \left(\bar{A}\bar{\cR}(x)[t]\right)_{1:k}
        +\bar{B}x[t+1]_{1:k}+\bar{\theta}_{1:k}+o(1)\\
        z_0\mathbf{1}_{d-k}
        +\delta\left(\bar{A}\bar{\cR}(x)[t]\right)_{k+1:d}
        +\delta\left(\bar{B}x[t+1]+\bar{\theta}\right)_{k+1:d}+o(\delta)
    \end{bmatrix}
    \\
    &=\begin{bmatrix}
        \bar{\cR}(x)[t+1]_{1:k}+o(1)\\
        \sigma'(z_0)\delta\bar{\cR}(x)[t+1]_{k+1:d}+o(\delta)
    \end{bmatrix}.
\end{align}
Setting $\cQ^\delta=\bar{Q}\begin{bmatrix}I_k& & \\ &\frac{1}{\sigma'(z_0)\delta}I_{d-k}\end{bmatrix}$ and choosing $\delta$ small enough complete the proof:
\begin{equation}
    \cQ^\delta\cR_2^\delta\cR_1^\delta\cP(x)[t]
    =\bar{Q}\begin{bmatrix}
    \bar{\cR}(x)[t]_{1:k}+o(1)\\
    \bar{\cR}(x)[t]_{k+1:d}+o(1)
    \end{bmatrix}
    =
    \bar{\cQ}\bar{\cR}(x)[t]+o(1)
    \rightarrow \bar{\cQ}\bar{\cR}(x)[t].
\end{equation}

\paragraph{Case 2: $\sigma(z_0)\ne 0$}~
\\
When $\sigma(z_0)\ne0$, there is $\sigma(z_0)$ term independent of $\delta$ in the Taylor expansion of $\sigma(z_0+\delta x)=\sigma(z_0)+\delta\sigma'(z_0) x+o(\delta)$.
An additional width removes the term in this case; hence we need a lifting map $\cP:\bR^{d\times N}\rightarrow\bR^{(d+1)\times N}$:
\begin{equation}
    \cP(x)[t]=\begin{bmatrix}x[t]\\0\end{bmatrix}.
\end{equation}
Now for $\delta>0$ define $\cR_1^\delta$ as
\begin{equation}
    \cR_1^{\delta}\left(X\right)\coloneqq\sigma\left(\delta\begin{bmatrix}\bar{B}& \\ &0\end{bmatrix}X
    +\delta\begin{bmatrix}\bar{\theta}\\0\end{bmatrix}+z_0\mathbf{1}_{d+1}\right).
\end{equation}
As in the previous case, we have
\begin{equation}
    \cR_1^{\delta}\circ\cP(x)[t+1]=\begin{bmatrix}
    \sigma(z_0)\mathbf{1}_d+\delta\sigma'(z_0)\left(\bar{B}x[t+1]+\bar{\theta}\right)+o\left(\delta\right)\\
    \sigma(z_0)
    \end{bmatrix},
\end{equation}
and construct a second cell $\cR_2^\delta$ to compute
\begin{equation}
    \label{eqn:second_rnn_cell_2}
    \cR_2^\delta(x)[t+1]=\sigma\left(
    \tilde{A}\cR_2^\delta(x)[t]
    +\frac{1}{\sigma'(z_0)}\begin{bmatrix}\delta^{-1}I_{k}& \\ &I_{d+1-k}\end{bmatrix}x[t+1]
    +\begin{bmatrix}-\frac{\sigma(z_0)}{\delta\sigma'(z_0)}\mathbf{1}_k\\z_0\mathbf{1}_{d+1-k}-\frac{\sigma(z_0)}{\sigma'(z_0)}\mathbf{1}_{d+1-k}\end{bmatrix}
    \right),
\end{equation}
where $\tilde{A}=\begin{bmatrix}I_k& \\ &\delta I_{d+1-k}\end{bmatrix}
\begin{bmatrix}\bar{A} & \\  & 0\end{bmatrix}
\begin{bmatrix}I_k& & \\ &\frac{1}{\delta\sigma'(z_0)}I_{d-k}&-\frac{1}{\delta\sigma'(z_0)}\mathbf{1}_{d-k}\\ & &0\end{bmatrix}$.\\
After that, the first output of $\cR_2^\delta\cR_1^\delta\cP(x)[t]$ becomes
\begin{align}
    \cR_2^\delta\cR_1^\delta\cP(x)[1]&
    =\sigma\left(\frac{1}{\sigma'(z_0)}\begin{bmatrix}\delta^{-1}I_k& \\ &I_{d+1-k}\end{bmatrix}\cR_1^\delta(x)[1]
    +\begin{bmatrix}-\frac{\sigma(z_0)}{\delta\sigma'(z_0)}\mathbf{1}_k\\z_0\mathbf{1}_{d+1-k}-\frac{\sigma(z_0)}{\sigma'(z_0)}\mathbf{1}_{d+1-k}\end{bmatrix}\right)
    \\
    &=\sigma\left(\begin{bmatrix}\left(\bar{B}x[1]+\bar{\theta}\right)_{1:k}+o(\delta)\\
    \left(z_0\mathbf{1}_{d-k}+\delta(\bar{B}x[1]+\bar{\theta}\right)_{k+1:d}+o(\delta)\\
    z_0\end{bmatrix}\right)
    \\
    &=\begin{bmatrix}\sigma\left(\bar{B}x[1]+\bar{\theta}\right)_{1:k}+o(1)\\
    \sigma(z_0)\mathbf{1}_{d-k}+\sigma'(z_0)\delta\left(\bar{B}x[1]+\bar{\theta}\right)_{k+1:d}+o(\delta)\\
    \sigma(z_0)
    \end{bmatrix}
    \\
    &=\begin{bmatrix}\bar{\cR}(x)[1]_{1:k}+o(1)\\
    \sigma(z_0)\mathbf{1}_{d-k}+\sigma'(z_0)\delta\bar{\cR}(x)[1]_{k+1:d}+o(\delta)\\
    \sigma(z_0)\end{bmatrix}.
\end{align}
Assume $\cR_2^\delta\cR_1^\delta\cP(x)$ and use mathematical induction on time $t$.
\begin{equation}
    \label{eqn:induction_hypothesis_2}
    \cR_2^\delta\cR_1^\delta\cP(x)[t]=\begin{bmatrix}\bar{\cR}(x)[t]_{1:k}+o(1)\\
    \sigma(z_0)\mathbf{1}_{d-k}+\sigma'(z_0)\delta\bar{\cR}(x)[t]_{k+1:d}+o(\delta)\\
    \sigma(z_0)\end{bmatrix}.
\end{equation}
Direct calculation yields
\begin{align}
    &\frac{1}{\sigma'(z_0)}\begin{bmatrix}\delta^{-1}I_k& \\ &I_{d+1-k}\end{bmatrix}\cR_1^\delta\cP(x)[t+1]
    +\begin{bmatrix}-\frac{\sigma(z_0)}{\delta\sigma'(z_0)}\mathbf{1}_{k}\\z_0\mathbf{1}_{d+1-k}-\frac{\sigma(z_0)}{\sigma'(z_0)}\mathbf{1}_{d+1-k}\end{bmatrix}
    \\
    &=\begin{bmatrix}\bar{B}x[t+1]_{1:k}+\bar{\theta}_{1:k}+o(1)\\z_0\mathbf{1}_{d-k}+\delta\left(\bar{B}x[t+1]+\bar{\theta}\right)_{k+1:d}+o(\delta)\\z_0\end{bmatrix},
\end{align}
and
\begin{align}
    &\tilde{A}\cR_2^\delta\cR_1^\delta\cP(x)[t]
    \\
    &=\tilde{A}
    \begin{bmatrix}
    \bar{\cR}(x)[t]_{1:k}+o(1)\\\sigma(z_0)\mathbf{1}_{d-k}+\sigma'(z_0)\delta\bar{\cR}(x)[t]_{k+1:d}+o(\delta)\\\sigma(z_0)
    \end{bmatrix}
    \\
    &=\begin{bmatrix}I_k& \\ &\delta I_{d+1-k}\end{bmatrix}
    \begin{bmatrix}\bar{A} &  \\   & 0\end{bmatrix}
    \begin{bmatrix}\bar{\cR}(x)[t]_{1:k} +o(1)\\ \bar{\cR}(x)[t]_{k+1:d}+o(1) \\ 0\end{bmatrix}
    \\
    &=\begin{bmatrix} \left(\bar{A}\bar{\cR}(x)[t]\right)_{1:k}+o(1) \\ \delta\left(\bar{A}\bar{\cR}(x)[t]\right)_{k+1:d}+o(\delta)\\0\end{bmatrix}.
\end{align}
Adding two terms in \eqref{eqn:second_rnn_cell_2}, we obtain the induction hypothesis \eqref{eqn:induction_hypothesis_2} for $t+1$,
\begin{equation}
    \cR_2^\delta\cR_1^\delta\cP(x)[t+1]=\begin{bmatrix}
    \bar{\cR}(x)[t+1]_{1:k}+o(1)\\
    \sigma(z_0)\mathbf{1}_{d-k}+\sigma'(z_0)\delta\bar{\cR}(x)[t+1]_{k+1:d}+o(\delta)\\
    \sigma(z_0)
    \end{bmatrix}.
\end{equation}
Setting $\cQ^\delta=\begin{bmatrix}\bar{Q}&0\end{bmatrix}\begin{bmatrix}I_k& & \\ &\frac{1}{\sigma'(z_0)\delta}I_{d-k}&-\frac{1}{\sigma'(z_0)\delta}\mathbf{1}_{d-k}\\ & &0\end{bmatrix}$ and choosing $\delta$ small enough complete the proof:
\begin{equation}
    \cQ^\delta\cR_2^\delta\cR_1^\delta\cP(x)[t]
    =\begin{bmatrix}\bar{Q}&0\end{bmatrix}\begin{bmatrix}
    \bar{\cR}(x)[t]_{1:k}+o(1)\\
    \bar{\cR}(x)[t]_{k+1:d}+o(1)\\
    0
    \end{bmatrix}
    =
    \bar{\cQ}\bar{\cR}(x)[t]+o(1)
    \rightarrow \bar{\cQ}\bar{\cR}(x)[t].
\end{equation}


\section{Proof of the Lemma \ref{lemma:linear_sum_into_rnn}}
\label{appendix:linear_sum_into_rnn}
It suffices to show that there exists a modified RNN $\cN$ that computes
\begin{equation}
    \cN(x)[N]=\begin{bmatrix}x[N]\\ \sum_{t=1}^NA[t]x[t]\end{bmatrix},    
\end{equation}
for given matrices $A[1],\ldots,A[N]\in\bR^{1\times d_x}$.

RNN should have multiple layers to implement the arbitrary linear combination.
To overcome the complex time dependency deriving from deep structures and explicitly formulate the results of deep modified RNN, we force $A$ and $B$ to use the information of the previous time step in a limited way.
Define the modified RNN cell at $l$-th layer $\cR_l$ as
\begin{equation}
    \label{eqn:technical_rnn_cell}
    \cR_l(x)[t+1]=A_l\cR_l(x)[t]+B_lx[t+1]
\end{equation}
where $A_l=\begin{bmatrix} O_{d_x,d_x}&O_{d_x,1}\\ O_{1,d_x}& 1\end{bmatrix}$, $B_l=\begin{bmatrix}I_{d_x}&O_{d_x,1}\\ b_l&1\end{bmatrix}$ for $b_l\in\bR^{1\times d_x}$.

Construct a modified RNN $\cN_L$ for each $L\in\bN$ as
\begin{equation}
    \cN_L\coloneqq \cR_L\circ\cR_{L-1}\circ\cdots\circ\cR_1,
\end{equation}
and denote the output of $\cN_L$ at each time $m$ for an input sequence $x'=\begin{bmatrix}x\\0\end{bmatrix}\in\bR^{d_x+1}$ of embedding of $x$:
\begin{equation}
    \label{eqn:binomial_coefficient_matrix}
    T(n,m)\coloneqq \cN_n\left(x'\right)[m].
\end{equation}
Then we have the following lemma.

\begin{lemma}
    \label{lemma:coeff_of_bixj}
    Let $T(n,m)$ be the matrix defined by \eqref{eqn:binomial_coefficient_matrix}. Then we have
    \begin{equation}
        T(n,m)=\begin{bmatrix}x[m]\\ \sum_{i=1}^\infty\sum_{j=1}^\infty
        \binom{n+m-i-j}{n-i}b_ix[j]\end{bmatrix},
    \end{equation}
    where $\binom{n}{k}$ means binomial coefficient $\frac{n!}{k!(n-k)!}$ for $n\ge k$.
    We define $\binom{n}{k}=0$ for the case of $k>n$ or $n<0$ for notational convenience.
\end{lemma}
\begin{proof}
    Since there is no activation in modified RNN \eqref{eqn:technical_rnn_cell}, $T(n,m)$ has the form of
    \begin{equation}
        \label{eqn:recurrence_relation}
        T(n,m)=\begin{bmatrix}x_m\\
        \sum_{i=1}^\infty\sum_{j=1}^\infty\alpha_{i,j}^{n,m}b_ix[j]\end{bmatrix}.
    \end{equation}
    From the definition of the modified RNN cell and $T$, we first show that $\alpha$ satisfies the recurrence relation
    \begin{equation}
        \alpha^{n,m}_{i,j}=\left\{\begin{array}{ll}\alpha_{i,j}^{n-1,m}+\alpha_{i,j}^{n,m-1}+1,&\text{if $n=i$ and $m=j$},\\
        \alpha_{i,j}^{n-1,m}+\alpha_{i,j}^{n,m-1},&\text{otherwise}\end{array}\right.  
    \end{equation}
    using mathematical induction on $n,m$ in turn.
    Initially,
    $T(0,m)=\begin{bmatrix}x_m\\0\end{bmatrix}$, $T(n,0)=\begin{bmatrix} O_{d_x,1}\\0\end{bmatrix}$ by definition, and \eqref{eqn:recurrence_relation} holds when $n=0$.
    Now assume \eqref{eqn:recurrence_relation} holds for $n\le N$, any $m$.
    To show that \eqref{eqn:recurrence_relation} holds for $n=N+1$ and any $m$, use mathematical induction on $m$.
    By definition, we have $\alpha_{i,j}^{n,0}=0$ for any $n$.
    Thus \eqref{eqn:recurrence_relation} holds when $n=N+1$ and $m=0$.
    Assume it holds for $n=N+1$ and $m\le M$.
    Then
    \begin{equation}
        \begin{aligned}
        &T(N+1,M+1)
        \\
        &=\begin{bmatrix}O_{d_x,d_x}&O_{d_x,1}\\O_{1,d_x}&1\end{bmatrix}
        \begin{bmatrix}x_M\\ \sum_{i=1}^\infty\sum_{j=1}^\infty\alpha_{i,j}^{N+1,M}x_j\end{bmatrix}
        +\begin{bmatrix}I_{d_x}&O_{d_x,1}\\ b_{N+1}&1\end{bmatrix}\begin{bmatrix}x_{M+1}\\ \sum_{i=1}^\infty\sum_{j=1}^\infty\alpha_{i,j}^{N,M+1}x_j\end{bmatrix}
        \\
        &=\begin{bmatrix}O_{d_x,1}\\\sum_{i=1}^\infty\sum_{j=1}^\infty\alpha_{i,j}^{N+1,M}b_ix_j\end{bmatrix}
        +\begin{bmatrix}x_{M+1}\\b_{N+1}x_{M+1}+\sum_{i=1}^\infty\sum_{j=1}^\infty\left(\alpha_{i,j}^{N+1,M}+\alpha_{i,j}^{N,M+1}\right)b_ix_j\end{bmatrix}.
        \end{aligned}
    \end{equation}
    Hence the relation holds for $n=N+1$ and any $m>0$.
    
    Now remains to show
    \begin{equation}
    \label{eqn:alpha_binom}
    \alpha_{i,j}^{n,m}=\left\{\begin{array}{ll}\binom{m+n-i-j}{n-i}&\text{if $1\le i\le n$, $1\le j\le m$}\\
                                                0&\text{otherwise}\end{array}\right.
    \end{equation}
    
    From the initial condition of $\alpha$, we know $\alpha^{0,m}_{i,j} = \alpha^{n,0}_{i,j}=0$ for all $n,m\in\bN$.
    After some direct calculation with the recurrence relation \eqref{eqn:recurrence_relation} of $\alpha$, we have
    \begin{enumerate}[i)]
        \item 
        If $n<i$ or $m<j$, $\alpha_{i,j}^{n,m}=0$ as $\alpha^{n,m}_{i,j} = \alpha^{n-1,m}_{i,j} + \alpha^{n,m-1}_{i,j}$.
        \item 
        $\alpha_{i,j}^{i,j}=\alpha_{i,j}^{i-1,j}+\alpha_{i,j}^{i,j-1}+1=1$.
        \item 
        $\alpha_{i,j}^{i,m}=\alpha_{i,j}^{i-1,m}+\alpha_{i,j}^{i,m-1}=\alpha_{i,j}^{i,m-1}$ implies $\alpha_{i,j}^{i,m}=1$ for $m>j$.
        \item 
        Similarly, $\alpha^{n,j}_{i,j} = \alpha^{n-1,j}_{i,j} + \alpha^{n,j-1}_{i,j} = \alpha^{n-1,j}_{i,j}$ implies $\alpha_{i,j}^{n,j}=1$ for $n>i$.
    \end{enumerate}
    Now use mathematical induction on $n+m$ starting from $n+m=i+j$ to show $\alpha^{n,m}_{i,j} = \binom{m+n-i-j}{n-i}$ for $n\ge i$, $m\ge j$.
    \begin{enumerate}[i)]
        \item 
        $n+m=i+j$ holds only if $n=i$, $m=j$ for $n\ge i$, $m\ge j$. In the case, $\alpha^{i,j}_{i,j}=1=\binom{m+n-i-j}{n-i}$.
        \item
        Assume that \eqref{eqn:alpha_binom} holds for any $n$, $m$ with $n+m=k$ as induction hypothesis.
        Now suppose $n+m=k+1$ for given $n$, $m$.
        If $n=i$ or $m=j$ we already know $\alpha_{i,j}^{n,m}=1=\binom{m+n-i-j}{n-i}$.
        Otherwise $n-1\ge i$, $m-1\ge j$, and we have
        \begin{equation}
            \begin{aligned}
                \alpha^{n,m}_{i,j} &= \alpha^{n-1,m}_{i,j} + \alpha^{n,m-1}_{i,j}
                \\ &= \binom{m+n-1-i-j}{n-1-i} + \binom{m+n-1-i-j}{n-i}
                \\ & = \binom{m+n-i-j}{n-i},
            \end{aligned}
        \end{equation}
        which completes the proof.
    \end{enumerate}
\end{proof}

We have computed the output of modified RNN $\cN_N$ such that
\begin{equation}
    \cN_N\left(x'\right)[N]=\begin{bmatrix}x[N]\\
    \sum_{i=1}^N\sum_{j=1}^N\binom{2n-i-j}{n-i}b_ix[j]\end{bmatrix}.
\end{equation}
If the square matrix $\Lambda_N=\left\{\binom{2n-i-j}{n-i}\right\}_{1\le i,j\le N}$ has inverse $\Lambda_N^{-1}=\left\{\lambda_{i,j}\right\}_{1\le i,j\le N}$, $b_i=\sum_{t=1}^N\lambda_{t,i}A[t]$ satisfies
\begin{equation*}
    \begin{aligned}
    \sum_{i=1}^N\sum_{j=1}^N\binom{2n-i-j}{n-i}b_ix[j]
    &=\sum_{i=1}^n\sum_{j=1}^n\sum_{t=1}^n\binom{2n-i-j}{n-i}\lambda_{t,i}A[t]x[j]
    \\
    &=\sum_{j=1}^n\sum_{t=1}^n\left[\sum_{i=1}^n\binom{2n-i-j}{n-i}\lambda_{t,i}\right]A[t]x[j]
    \\
    &=\sum_{j=1}^n\sum_{t=1}^n\delta_{j,t}A[t]x[j]
    \\
    &=\sum_{j=1}^n A[j]x[j],
    \end{aligned}
\end{equation*}
where $\delta$ is the Kronecker delta function.

The following lemma completes the proof.
\begin{lemma}
\label{lemma:the_matrix}
Matrix $\Lambda_n=\left\{\binom{2n-i-j}{n-i}\right\}_{1\le i,j\le n}\in\bR^{n\times n}$ is invertible.
\end{lemma}
\begin{proof}
    Use mathematical induction on $n$.
    $\Lambda_1$ is a trivial case.
    Assume $\Lambda_n$ is invertible.
    \begin{equation}
        \begin{aligned}
        \Lambda_{n+1} = \begin{bmatrix} 
        \binom{2n}{n} & \binom{2n-1}{n} & \binom{2n-2}{n}&\dots& \binom{n+1}{n}& \binom{n}{n}\\
        \binom{2n-1}{n-1} & \binom{2n-2}{n-1} & \binom{2n-3}{n-1}&\dots &\binom{n}{n-1}& \binom{n-1}{n-1}\\
        \vdots &\vdots &\vdots & \ddots & \vdots& \vdots \\
        \binom{n+1}{1} & \binom{n}{1} & \binom{n-1}{1}&\dots& \binom{2}{1}& \binom{1}{1}\\
        \binom{n}{0} & \binom{n-1}{0} & \binom{n-2}{0}&\dots& \binom{1}{0}& \binom{0}{0}
        \end{bmatrix}.
        \end{aligned}
    \end{equation}
    Applying elementary row operation to $\Lambda_{n+1}$ by multiplying the matrix $E$ on the left and elementary column operation to $E\Lambda_{n+1}$ by multiplying the matrix $E^T$ on the right where
    \begin{equation}
        E=\begin{bmatrix} 
        1 & -1 & 0& \dots& 0&  0\\
        0 & 1 & -1 &\dots&0&  0\\
        0 & 0 & 1 &\dots& 0&  0\\
        \vdots & \vdots & \vdots& \ddots & \vdots& \vdots \\
        0 & 0& 0 &\dots &1&  -1\\
        0 & 0&0& \dots      &0 & 1 
        \end{bmatrix},
    \end{equation}
    we obtain the following relation:
    \begin{equation}
        E\Lambda_{n+1}  E^{T}= \begin{bmatrix} 
        \binom{2n-2}{n-1} & \binom{2n-3}{n-1} & \binom{2n-4}{n-1}&\dots& \binom{n-1}{n-1}& 0\\
        \binom{2n-3}{n-2} & \binom{2n-4}{n-2} & \binom{2n-5}{n-2}&\dots &\binom{n-2}{n-2}& 0\\
        \vdots &\vdots &\vdots & \ddots & \vdots& \vdots \\
        \binom{n-1}{0} & \binom{n-2}{0} & \binom{n-3}{0}&\dots& \binom{0}{0}& 0\\
        0 & 0 & 0&\dots& 0& 1
        \end{bmatrix}
        = \begin{bmatrix} \Lambda_{n}& O_{n,1}\\ O_{1,n}& 1\end{bmatrix}.
    \end{equation}
    Hence $\Lambda_{n+1}$ is invertible by the induction hypothesis.
\end{proof}

\begin{corollary}
\label{cor:lambda_is_full_rank}
    The following matrix $\Lambda_{n,k}\in\bR^{k\times n}$ is full-rank.
    \begin{equation}
        \Lambda_{n,k}=\left\{\binom{2n-i-j}{n-i}\right\}_{n-k+1 \leq i \leq n, 1 \leq j \leq n }.
    \end{equation}
\end{corollary}
We will use the matrix $\Lambda_{n,k}$ in the proof of Lemma \ref{lemma:linear_sum_into_trnn} to approximate a sequence-to-sequence function.


\section{Proof of Lemma \ref{lemma:linear_sum_into_trnn}}
\label{appendix:linear_sum_into_trnn}
Define token-wise lifting map $\cP:\bR^{d_x}\rightarrow\bR^{d_x+1}$ and modified TRNN cell $\cT\cR_l:\bR^{(d_x+1)\times N}\rightarrow\bR^{(d_x+1)\times N}$ as in the proof of Lemma \ref{lemma:linear_sum_into_rnn}:
\begin{align}
    \label{eqn:special_trnn_in_appendix}
    \cP(x)[t]&=\begin{bmatrix}x[m]\\0\end{bmatrix},
    \\
    \cT\cR_l(X)[t+1]&=A_l\cT\cR_l(X)[t]+B_l[t](X)[t+1],
\end{align}
where $A_l=\begin{bmatrix} O_{d_x,d_x}&O_{d_x,1}\\ O_{1,d_x}& 1\end{bmatrix}$, $B_l[t]=\begin{bmatrix}I_{d_x}&O_{d_x,1}\\ b_l[t]&1\end{bmatrix}$ for $b_l[t]\in\bR^{1\times d_x}$.
Then we have
\begin{equation}
    \begin{aligned}
        T(n,m)&\coloneqq\cN_n\left(x\right)[m]\\
        &=\begin{bmatrix}x[m]\\ \sum_{i=1}^\infty\sum_{j=1}^\infty \binom{n+m-i-j}{n-i}b_i[j]x[j]\end{bmatrix},
    \end{aligned}
\end{equation}
where $x\in\bR^{d_x\times N}$ and $\cN_L=\cT\cR_L\circ\cT\cR_{L-1}\circ\cdots\circ\cT\cR_1\circ\cP$.

Since for each $t$, the matrix
\begin{equation}
    \Lambda_{N,N-t+1}=\left\{\binom{2N-i-j}{N-i}\right\}_{t\le i\le N,1\le j\le N}=\left\{\binom{2N-t+1-i-j}{N-j}\right\}_{1\le i\le N-t+1,1\le j\le N}
\end{equation}
is full-rank, there exist $b_1[t],b_2[t],\ldots,b_N[t]$ satisfying 
\begin{equation}
    \Lambda_{N,N-t+1}\begin{bmatrix}
    b_1[t]\\ \vdots \\b_N[t]
    \end{bmatrix}=\begin{bmatrix}
    A_t[N]\\ \vdots\\ A_t[t]
    \end{bmatrix},
\end{equation}
or
\begin{equation}
    \sum_{j=1}^N\binom{N+k-j-t}{N-j}b_j[t]=A_t[k],
\end{equation}
for each $k=1,2,\ldots,N$. 
Then we obtain
\begin{equation}
    \begin{aligned}
    T(N,t)
    &=\sum_{i=1}^\infty\sum_{j=1}^\infty \binom{N+t-i-j}{N-i}b_i[j]x[j]\\
    &=\sum_{j=1}^t\sum_{i=1}^N\binom{N+t-i-j}{N-i}b_i[j]x[j]\\
    &=\sum_{j=1}^tA_j[t]x[j].
    \end{aligned}
\end{equation}


\section{Proof of Lemma \ref{lemma:rnn_approximate_trnn}}
\label{appendix:rnn_approximate_trnn}
As one of the modified TRNNs that computes \eqref{eqn:linear_sum_into_trnn}, we use the modified TRNN defined in Appendix \ref{appendix:linear_sum_into_trnn}.
Specifically, we show that for a given $l$, there exists a modified RNN of width $d_x+2+\gamma(\sigma)$ that approximates the modified TRNN cell $\cT\cR_l:\bR^{(d_x+1)\times N}\rightarrow\bR^{(d_x+1)\times N}$ defined by \eqref{eqn:special_trnn_in_appendix}.
Suppose $K\subset\bR^{d_x}$, $K'\subset\bR$ are compact sets and $X\in \left(K\times K'\right)^N\subset\bR^{(d_x+1)\times N}$.
Then the output of the TRNN cell $\cT\cR_l$ is
\begin{equation}
    \cT\cR_l(X)[t]=\begin{bmatrix}X[t]_{1:d_x} \\ \sum_{j=1}^t b_l[j]X[j]_{1:d_x}+\sum_{j=1}^t X[j]_{d_x+1}\end{bmatrix}.
\end{equation}
Without loss of generality, assume $K\subset \left[0,\frac{1}{2}\right]^{d_x}$ and let $\gamma=\gamma(\sigma)$.
Let $\cP:\bR^{d_x+1}\rightarrow\bR^{d_x+2+\gamma}$ be a token-wise linear map defined by $\cP(X)=\begin{bmatrix}X_{1:d_x} \\ 0 \\ X_{d_x+1} \\ \mathbf{0}_{\gamma}\end{bmatrix}$.
Construct the modified recurrent cells $\cR_1,\cR_2:\bR^{(d_x+2+\gamma(\sigma))\times N}\rightarrow\bR^{(d_x+2+\gamma(\sigma))\times N}$ as for $X'\in\bR^{(d_x+2+\gamma)\times N}$,
\begin{align}
    \cR_1(X')[t+1]
    &=\begin{bmatrix}
        O_{d_x,d_x}& &\\
        &1&&\\
        && O_{1+\gamma,1+\gamma}
    \end{bmatrix}
    \cR_1(X')[t]+X'[t+1]
    +\begin{bmatrix}
        \mathbf{0}_{d_x}\\1\\ \mathbf{0}_{1+\gamma}
    \end{bmatrix},
    \\
    \cR_2(X')[t+1]
    &=\begin{bmatrix}
        I_{d_x}&\mathbf{1}_{d_x}&&\\
        &\reviewB{1}&&\\
        &&\reviewB{1}&\\
        &&& O_{\gamma,\gamma}
    \end{bmatrix}
    X'[t].
\end{align}
Then, by definition for $X\in\left(K\times K'\right)^N$,
\begin{equation}
    \cR_2\cR_1\cP(X)[t]=\begin{bmatrix} X[t]_{1:d_x}+t\mathbf{1}_{d_x} \\ t \\ X[t]_{d_x+1} \\ \mathbf{0}_{\gamma}\end{bmatrix}.
\end{equation}
Note that $D_i=\{\cR_2\cR_1\cP(X)[i]_{1:d_x}\mid X\in \left(K\times K'\right)^N \}=\{X[i]_{1:d_x}+t\mathbf{1}_{d_x}\mid X\in \left(K\times K'\right)^N \}$ are disjoint each other, $D_i\cap D_j=\phi$ for all $i\ne j$.

By the universal approximation theorem of deep MLP from \cite{hanin2017approximating, Kidger_2020}, for any $\delta_l>0$, there exists an MLP $\cN_{l,MLP}:\bR^{d_x}\rightarrow\bR^{d_x+1}$ of width $d_x+1+\gamma$ such that for $v\in \bR^{d_x}$,
\begin{align}
    &\cN_{l,MLP}(v)_{1:d_x}=v
    \\
    &\sup_{t=1,\ldots,N}\sup_{v\in D_t}\left\|b_l[t]\left(v-t\mathbf{1}{d_x}\right) - \cN_{l,MLP}(v)_{d_x+1}\right\|<\delta_l.
\end{align}
Since token-wise MLP is implemented by RNN with the same width, there exists an RNN $\cN_l:\bR^{d_x+2+\gamma}\rightarrow\bR^{d_x+2+\gamma}$ of width $d_x+2+\gamma$ whose components all but $(d_x+2)$-th construct $\cN_{l,MLP}$ so that for all 
$X'\in\bR^{d_x+2+\gamma}$,
\begin{equation}
    \cN_l(X')[t]
    =\begin{bmatrix}
       \cN_{l,MLP}\left(X'[t]_{1:d_x}\right) \\ X'[t]_{d_x+2} \\ \mathbf{0}_{\gamma}
    \end{bmatrix}.
\end{equation}
Then for $X\in\left(K\times K'\right)^N$ we have
\begin{equation}
    \cN_l\cR_2\cR_1\cP(X)[t]
    =\begin{bmatrix}\cN_{l,MLP}\left(X[t]_{1:d_x}+t\mathbf{1}_{d_x} \right) \\ X[t]_{d_x+1} \\ \mathbf{0}_{\gamma}\end{bmatrix}.
\end{equation}
Finally, define a recurrent cell $\cR_3:\bR^{d_x+2+\gamma}\rightarrow\bR^{d_x+2+\gamma}$ of width $d_x+2+\gamma$ as
\begin{equation}
    \cR_3(X')[t+1]
    =\begin{bmatrix}O_{d_x+1,d_x+1}&&\\ &1&& \\ &&O_{\gamma,\gamma}\end{bmatrix}\cR_3(X')[t]
    +\begin{bmatrix}I_{d_x} &&& \\ &1&&\\ &1&1& \\ &&&O_{\gamma,\gamma} \end{bmatrix}
    X'[t+1],
\end{equation}
and attain
\begin{equation}
    \cR_3\cN_{1}\cR_2\cR_1\cP(X)[t]
    =\begin{bmatrix}
        X[t]_{1:d_x}+t\mathbf{1}_{d_x}
        \\
        \cN_{l,MLP}\left(X[t]_{1:d_x}+t\mathbf{1}_{d_x}\right)_{d_x+1}
        \\
        \sum_{j=1}^t \cN_{l,MLP}\left(X[j]_{1:d_x}+j\mathbf{1}_{d_x}\right)_{d_x+1} + \sum_{j=1}^t X[j]_{d_x+1}
        \\
        \mathbf{0}_{\gamma}
    \end{bmatrix}.
\end{equation}
With the token-wise projection map $\cQ:\bR^{d_x+2+\gamma}\rightarrow\bR^{d_x+1}$ defined by $\cQ(X')=\begin{bmatrix}X'_{1:d_x}\\X'_{d_x+2}\end{bmatrix}$, an RNN $\cQ\cR_3\cN_l\cR_2\cR_1\cP:\bR^{(d_x+1)\times N}\rightarrow\bR^{(d_x+1)\times N}$ of width $d_x+2+\gamma$ maps $X\in\bR^{(d_x+1)\times N}$ to
\begin{equation}
    \cQ\cR_3\cN_l\cR_2\cR_1\cP(X)[t] = \begin{bmatrix}X[t]_{1:d_x}+t\mathbf{1}_{d_x}\\ \sum_{j=1}^t \cN_{l,MLP}\left(X[j]_{1:d_x}+j\mathbf{1}_{d_x}\right)_{d_x+1} + \sum_{j=1}^t X[j]_{d_x+1}\end{bmatrix}.
\end{equation}
Since $\cN_{l,MLP}\left(X[j]_{1:d_x}+j\mathbf{1}_{d_x}\right)_{d_x+1}\rightarrow b_l[j]X[j]_{1:d_x}$, we have
\begin{equation}
    \sup_{X\in \left(K\times K'\right)^N}\left\|\cT\cR_l(X)-\cQ\cR_3\cN_l\cR_2\cR_1\cP(X)\right\| \rightarrow 0,
\end{equation}
as $\delta_l\rightarrow0$.
Approximating all $\cT\cR_l$ in Appendix \ref{appendix:linear_sum_into_trnn} finishes the proof.


\section{Proof of Lemma \ref{lemma:linear_sum_into_tbrnn}}
\label{appendix:linear_sum_into_tbrnn}
The main idea of the proof is to separate the linear sum  $\sum_{j=1}^NA_j[t]x[j]$ into the past-dependent part $\sum_{j=1}^{t-1}A_j[t]x[j]$ and the remainder part $\sum_{j=t}^N A_j[t]x[j]$.
Then, we construct modified TBRNN with $2N$ cells; the former $N$ cells have only a forward recurrent cell to compute the past-dependent part, and the latter $N$ cells have only a backward recurrent cell to compute the remainder.

Let the first $N$ modified TRNN cells $\cR_l:\bR^{(d_x+1)\times N}\rightarrow\bR^{(d_x+1)\times N}$ for $1\le l\le N$ be defined as in the proof of Lemma \ref{lemma:linear_sum_into_trnn}:
\begin{equation}
    \cR_l\left(X\right)[t+1]=A_l\cR_l\left(X\right)[t]+B_l[t]X[t+1],
\end{equation}
where $A_l=\begin{bmatrix} O_{d_x,d_x}&O_{d_x,1}\\ O_{1,d_x}& 1\end{bmatrix}$, $B_l[t]=\begin{bmatrix}I_{d_x}&O_{d_x,1}\\ b_l[t]&1\end{bmatrix}$ for $b_l[t]\in\bR^{1\times d_x}$.
Then, with token-wise lifting map $\cP:\bR^{d_x}\rightarrow\bR^{d_x+1}$ defined by $\cP(x)=\begin{bmatrix}x\\0\end{bmatrix}$, we construct modified TRNN  $\cN:\cR_N\circ\cdots\circ\cR_1\circ\cP:\bR^{d_x\times N}\rightarrow\bR^{(d_x+1)\times N}$.
We know that if $C_i[m]\in\bR^{1\times d_x}$ are given for $1\le m\le N$ and $1\le i\le m$, there exist $b_l[t]$ for $1\le l\le N$, such that
\begin{equation}
    \cN_N(x)[m]=\begin{bmatrix}x[m]\\ \sum_{i=1}^m C_i[m]x[i]\end{bmatrix}.
\end{equation}
Therefore, we will determine $C_i[m]$ after constructing the latter $N$ cells.
Let $f_m=\sum_{i=1}^mC_i[m]x[i]$ for brief notation.

After $\cN_N$, construct $N$ modified TRNN cells $\bar{\cR}_l:\bR^{(d_x+1)\times N}\rightarrow\bR^{(d_x+1)\times N}$ for $1\le l\le N$ in reverse order:
\begin{equation}
    \bar{\cR}_l\left(\bar{X}\right)[t-1]=\bar{A}_l\bar{\cR}_l\left(\bar{X}\right)[t]+\bar{B}_l[t]\bar{X}[t-1],
\end{equation}
where $\bar{A}_l=\begin{bmatrix} O_{d_x,d_x}&O_{d_x,1}\\ O_{1,d_x}& 1\end{bmatrix}$, $\bar{B}_l[t]=\begin{bmatrix}I_{d_x}&O_{d_x,1}\\ \bar{b}_l[t]&1\end{bmatrix}$ for $\bar{b}_l[t]\in\bR^{1\times d_x}$.
Define $\bar{\cN}_N=\bar{\cR}_N\circ\cdots\circ\bar{\cR}_1$, and we obtain the following result after a similar calculation with input sequence $\bar{X}[t]=\cN_N(x)[t]=\begin{bmatrix} x[t]\\f_t\end{bmatrix}$:
\begin{align}
    &\bar{\cN}_N\left(\bar{X}\right)[N+1-t]
    \\
    &=\begin{bmatrix}x[N+1-t]\\ \sum_{j=1}^t\left[\sum_{i=1}^N\binom{N+t-i-j}{N-i}\bar{b}_i[N+1-j]x[N+1-j]+\binom{N+t-1-j}{N-1}f_{N+1-j}\right]\end{bmatrix}.
\end{align}
We want to find $f_m$ and $\bar{b}_i[m]$ so that
\begin{equation}
    \label{eqn:wanted_result_brnn}
    \bar{\cN}_N\left(\bar{X}\right)[N+1-t]_{d_x+1}
    =
    \sum_{i=1}^NA_i[N+1-t]x[i],
\end{equation}
for each $t=1,2,\ldots, N$.

Note that $\sum_{j=1}^t \sum_{i=1}^N\binom{N+t-i-j}{N-i}\bar{b}_i[N+1-j]x[N+1-j]$ does not contain $x[1],x[2],\ldots, x[N-t]$ terms, so $\sum_{j=1}^t\binom{N+t-1-j}{N-1}f_{N+1-j}$ should contain $\sum_{i=1}^{N-t}A_i[N+1-t]x[i]$.
\begin{align}
    \sum_{j=1}^t\binom{N+t-1-j}{N-1}f_{N+1-j}
    &=\sum_{j=1}^t\binom{N+t-1-j}{N-1}\sum_{i=1}^{N+1-j}C_i[N+1-j]x[i]
    \\
    &=\sum_{j=1}^t\sum_{i=1}^{N+1-j}\binom{N+t-1-j}{N-1}C_i[N+1-j]x[i]
    \\
    &=\sum_{i=N+2-t}^N\sum_{j=1}^{N+1-i}\binom{N+t-1-j}{N-1}C_i[N+1-j]x[i]
    \\
    &\qquad + \sum_{i=1}^{N+1-t}\sum_{j=1}^t\binom{N+t-1-j}{N-1}C_i[N+1-j]x[i].
\end{align}
Since matrix $\Lambda_i=\left\{\binom{N+t-1-j}{N-1}\right\}_{1\le t\le N+1-i, 1\le j\le N+1-i}$ is a lower triangular $(N+1-i)\times (N+1-i)$ matrix with unit diagonal components, there exist $C_i[i],C_i[i+1],\ldots,C_i[N]$ such that
\begin{equation}
    \sum_{j=1}^{t}\binom{N+t-1-j}{N-1}C_i[N+1-j] = A_i[N+1-t],
\end{equation}
for each $t=1,2,\ldots, N+1-i$.

We now have 
\begin{align}
    &\sum_{j=1}^t\binom{N+t-1-j}{N-1}f_{N+1-j}
    \\
    &=\sum_{i=N+2-t}^N\sum_{j=1}^{N+1-i}\binom{N+t-1-j}{N-1}C_i[N+1-j]x[i]
    +\sum_{i=1}^{N+1-t}A_i[N+1-t]x[i]
    \\
    &=\sum_{i=1}^{t-1}\sum_{j=1}^{i}\binom{N+t-1-j}{N-1}C_{N+1-i}[N+1-j]x[N+1-i]
    +\sum_{i=1}^{N+1-t}A_i[N+1-t]x[i]
    \\
    &=\sum_{j=1}^{t-1}\sum_{i=1}^{j}\binom{N+t-1-i}{N-1}C_{N+1-j}[N+1-i]x[N+1-j]
    +\sum_{i=1}^{N+1-t}A_i[N+1-t]x[i].
\end{align}
We switch $i$ and $j$ for the last equation.
By Corollary \ref{cor:lambda_is_full_rank}, there exist $\bar{b}_i[N+1-j]$ satisfying 
\begin{align}
    &\sum_{i=1}^N\binom{N+t-i-j}{N-i}\bar{b}_i[N+1-j]
    \\
    &=A_{N+1-j}[N+1-t]
    -\sum_{i=1}^j\binom{N+t-1-i}{N-1}C_{N+1-j}[N+1-i],
    \intertext{for $j=1,2,\ldots,t-1$, and}
    &\sum_{i=1}^N\binom{N+t-i-j}{N-i}\bar{b}_i[N+1-j]
    \\
    &=A_{N+1-j}[N+1-t],
\end{align}
for $j=t$.\\
With the above $C_i[m]$ and $\bar{b}_i[m]$, equation \eqref{eqn:wanted_result_brnn} holds for each $t=1,2,\ldots, N$.
It remains to construct modified TRNN cells to implement $f_m$, which comes directly from the proof of Lemma \ref{lemma:linear_sum_into_trnn}.

\newpage
\bibliography{Reference}

\end{document}